\documentclass{article}

\usepackage{amsmath,amsfonts,bm}

\def\eqref#1{equation~\ref{#1}}

\def\1{\bm{1}}

\DeclareMathAlphabet{\mathsfit}{\encodingdefault}{\sfdefault}{m}{sl}
\SetMathAlphabet{\mathsfit}{bold}{\encodingdefault}{\sfdefault}{bx}{n}

\usepackage{PRIMEarxiv}
\usepackage{amsthm}
\usepackage{bm}
\newtheorem{lemma}{Lemma}
\newtheorem{assumption}{Assumption}
\newtheorem{remark}{Remark}
\newtheorem{proposition}{Proposition}
\newtheorem{definition}{Definition}
\usepackage{enumitem}
\usepackage{wrapfig}
\usepackage{amssymb}

\usepackage[utf8]{inputenc} 
\usepackage[T1]{fontenc}    
\usepackage{hyperref}       
\usepackage{url}            
\usepackage{booktabs}       
\usepackage{amsfonts}       
\usepackage{nicefrac}       
\usepackage{microtype}      
\usepackage{lipsum}
\usepackage{fancyhdr}       
\usepackage{graphicx}       
\graphicspath{{media/}}     
\usepackage{natbib}
\usepackage{xcolor} 
\usepackage{authblk}

\title{Towards the Effect of Examples on In-Context Learning: A Theoretical Case Study}

\author[1]{Pengfei He}
\author[1]{Yingqian Cui}
\author[2]{Han Xu}
\author[1]{Hui Liu}
\author[3]{Makoto Yamada}
\author[1]{Jiliang Tang}
\author[4]{Yue Xing}

\affil[1]{Department of Computer Science and Engineering, Michigan State University}
\affil[2]{Department of Electrical and Computer Engineering, University of Arizona}
\affil[3]{Machine Learning and Data Science Unit, Okinawa Institute of Science and Technology (OIST)}
\affil[4]{Department of Statistics and Probability, Michigan State University}

\newtheorem{theorem}{Theorem}

\newcommand{\blue}[1]{{\color{blue}#1}}

\allowdisplaybreaks

\begin{document}
\maketitle

\begin{abstract}
In-context learning (ICL) has emerged as a powerful capability for large language models (LLMs) to adapt to downstream tasks by leveraging a few (demonstration) examples. Despite its effectiveness, the mechanism behind ICL remains underexplored. 
To better understand how ICL integrates the examples with the knowledge learned by the LLM during pre-training (i.e., pre-training knowledge) and how the examples impact ICL, this paper conducts a theoretical study in binary classification tasks.
In particular, we introduce a probabilistic model extending from the Gaussian mixture model to exactly quantify the impact of pre-training knowledge, label frequency, and label noise on the prediction accuracy.
Based on our analysis, when the pre-training knowledge contradicts the knowledge in the examples, whether ICL prediction relies more on the pre-training knowledge or the examples depends on the number of examples. 
In addition, the label frequency and label noise of the examples both affect the accuracy of the ICL prediction, where the minor class has a lower accuracy, and how the label noise impacts the accuracy is determined by the specific noise level of the two classes.
Extensive simulations are conducted to verify the correctness of the theoretical results, and real-data experiments also align with the theoretical insights.
Our work reveals the role of pre-training knowledge and examples in ICL, offering a deeper understanding of LLMs' behaviors in classification tasks.
\end{abstract}

\section{Introduction}

Large language models (LLMs) have revolutionized various fields, such as GitHub Copilot for software development, Microsoft 365 Copilot to embrace productivity and medical applications such as Med-Palm~\citep{singhal2023large}. 
A particularly intriguing capability of LLMs is in-context learning (ICL), discovered by \citep{brown2020language}, where LLMs can adapt to downstream tasks using only a few examples at the inference stage without altering model parameters.

Following this discovery, various studies have been conducted to evaluate and understand the ICL capability of LLMs. Empirically, \citep{garg2022can} demonstrates through simulation studies that transformer-based models can learn linear functions in context and generalize to data with unseen distributions. Later, \citep{hendel2023context, todd2023function} find that transformer models can encode input-output relationships within the hidden space of attention layers. \citep{min2022rethinking} observes that the key aspects of the demonstration include the label space, the input distribution, and the format of the prompt. Other studies also highlight the importance of example order and template usage \citep{liu2021makes, lu2021fantastically, wu2022self}, with further work focused on selecting examples \citep{lu2021fantastically,wang2024large,zhang2022active} or designing prompts to enhance ICL performance \citep{zhao2021calibrate, wei2022chain}.

Theoretically, research has explored two main directions to understand ICL. On one hand, many studies utilize the simplified transformer architecture and track its exact behavior.
For example, \citep{zhang2023trained} explicitly derives that a single linear self-attention layer trained by
gradient flow results in a competitive prediction error with the best linear predictor during ICL. \citep{von2023transformers, dai2022can} shows that one attention layer can be exactly constructed to perform gradient descent. 
\citep{han2023explaining} show that ICL asymptotically converges to kernel regression as the number of examples increases. 
Besides studies on one attention layer, other works also consider the scenario with multiple attention layers, in which the transformer works as a gradient descent to refine the prediction. \citep{akyurek2022learning} shows the existence of a transformer which can optimize the ridge-regression objective, and
\citep{ahn2024transformers,mahankali2023one} prove a transformer trained on noisy linear regression task minimizing the pre-training loss will implement gradient descent algorithm on examples.

On the other hand, Bayesian inference is also used to explain the emergence of ICL. \citep{xie2021explanation} first leverages a Hidden Markov Model to represent the pre-training data and proves that a transformer trained on such data exhibits the ICL ability.  
Later, \citep{jeon2024information} introduces an information-theoretic tool to show how ICL prediction error decays in the number and length of examples. 
Besides, based on \citep{lin2024dual}, LLMs typically demonstrate two key abilities \cite{lin2024dual}: retrieving knowledge from the pre-training data and learning from the examples in the prompt. They introduce a probabilistic model to understand {two modes of ICL on the linear regression tasks: task learning where predictions are made solely based on examples, and task retrieval where the pre-training knowledge is mainly used for prediction.} 

Despite these advancements, the understanding LLMs' ICL capability for classification tasks still remains limited. 
First, previous works cannot draw a consensus on certain behaviors of ICL. 
For example, \citep{min2022rethinking} empirically observes that injecting random noise to the labels of the examples does not hurt the ICL performance. They conjecture that this is because the pre-training knowledge dominates ICL. On the other hand, other studies, e.g., \cite{yoo2022ground,lin2024dual} draw different conclusions. Based on \cite{lin2024dual}, when taking a large number of examples (a large example size), the ICL will favor the knowledge in the examples. 
Therefore, a theoretical understanding of the role of label quality, the difference between pre-training and example knowledge, and the example size, is urgently needed. 
Second, existing theoretical frameworks are not sufficient to explain the observations in classification tasks. For instance, 
the balance of the example size in different classes matters in classification, while there is no such concept in regression. 

The above gaps drive the need for a theoretical exploration of how LLMs utilize pre-training knowledge and specific examples in ICL in classification scenarios. In particular, we aim to answer: \textbf{How do LLMs make predictions in classification tasks using their pre-training knowledge and examples?}

We explore the above question by conducting an exact theoretical analysis in a binary classification task. 
Technically, an LLM is firstly pre-trained using a large amount of pre-training data, and then at the inference stage, we feed the pre-trained LLM with a new data to make a prediction. Following the formulation of \cite{zhang2023trained,huang2023context,lin2024dual,han2023explaining}, we form prompts in the pre-training data and the testing data as $((x_1,y_1),(x_2,y_2),\ldots,(x_k,y_k),x_{query})$. The training objective is to predict the label $y_{query}$ for $x_{query}$\footnote{To simplify the notation, we use $(x,y)$ and $\{(x_i,y_i)\}$ to refer to $(x_{query},y_{query})$ and $\{(x_i,y_i)\}_{i\in[k]}$ respectively when no confusion arises.}, i.e., performing ICL using the examples $\{(x_i,y_i)\}_{i\in[k]}$, where $x_i\in \mathbb{R}^m$. Different from existing literature which considers regression tasks, we study the binary classification task, i.e., $y_i\in\{-1,+1\}$. Each training prompt samples different $x_i$s, and the relationship between $y_i$ and $x_i$ also differ among prompts, i.e., the model is trained to gain the ICL ability.
At the inference stage, we perform ICL on a test input $x$ by predicting its label given in-context examples, i.e. $\hat{y}_{ICL}=M(\{x_i,y_i\},x)$ where $M$ denotes the pre-trained LLM. 
Given the above definitions, in this paper, we study the effect of the pre-training data and the examples on the
distribution of $\hat{y}_{ICL}$, i.e. $P(\hat{y}_{ICL}=s|x, \{x_i,y_i\}_{i=1,\ldots,k})$ for $s\in \{+1,-1\}$, and the prediction accuracy is $P(\hat{y}_{ICL}=y)$. 

A central challenge in the analysis lies in the formulation and integration of both priors of pre-training knowledge and examples into a single closed-form ICL prediction. We accommodate this by considering a Gaussian mixture model as the data generation model (Section \ref{sec:assumption}), which is determined by three important parameters: means for the positive and negative class $\theta_+=\mathbb{E}[x_i|y_i=+1]$, $\theta_-=\mathbb{E}[x_i|y_i=-1]$, and $\pi=P(y_i=+1)$ as the expected fraction of positive samples. Then we distinguish the distributions of pre-training data, in-context examples and test data by considering different distributions of $\theta_+,\theta_-,\pi$. Given the formulations, we further derive a closed-form expression for ICL decision boundary given both the examples and pre-training knowledge (shown in Lemma \ref{lem:decision}).

Utilizing the decision boundary obtained from Lemma \ref{lem:decision}, our main results (\textbf{Theorem \ref{thm:icl_acc}, Proposition \ref{prop:asymp}}) quantify the accuracy as the interplay between pre-training knowledge and examples.
To be more specific, the accuracy $P(\hat{y}_{ICL}=y)$ is determined by the posterior of $\theta_+,\theta_-$. These posteriors are in fact mixtures of examples and pre-training knowledge. 
Built upon these general results, the phenomena of some specific demonstration scenarios can be explained well, including contradicting knowledge where examples carry different knowledge from pre-training data (\textbf{Proposition \ref{prop:contradict}}), label noise where the labels of examples are potentially flipped (\textbf{Proposition \ref{prop:noise}}) and imbalanced cases when the fraction of examples from both classes are not equal (\textbf{Proposition \ref{prop:imbalance}}). 

Besides, when conducting simulations to verify the above theoretical insights, we surprisingly reveal another counter-intuitive behavior when the examples are not selected independently: We fix exactly 50\% positive labels in each prompt in pre-training and provide only positive examples in the test prompt, then the ICL prediction is a firm negative. Intuitively, in this setting, $\pi$ is fixed to 0.5 in the pre-training, and the model memorizes this value. At the inference stage, when performing ICL, the model tries to match the ICL prediction so that the fraction of positive examples in the testing prompt is also close to 0.5. As a result, when all examples are positive, the ICL gives a firm negative prediction. We name this phenomenon as Mean Reversion and provide more explanations and theoretical justification (\textbf{Theorem \ref{them: dependent}}) to explain this behavior. This finding can help understand how LLMs consider dependency among tokens/sequences. 

\section{Classification Analysis with i.i.d. Examples} \label{sec:analysis independent}

In this section, we focus on the case where $(x_i,y_i)$s are i.i.d. samples from the same distribution within each prompt. To analyze the ICL performance, we first introduce the data generation model and data assumptions in Section \ref{sec:assumption}, then derive the ICL accuracy under general situations in Section  \ref{sec:decision}. We finally examine how examples influence ICL under specific demonstration scenarios (Section \ref{sec:demonstration}).

\subsection{Setups}\label{sec:assumption}

There are two steps in generating $\{(x_i,y_i)\}$ and $(x,y)$. First,
we follow the idea of Bayesian inference to impose a prior distribution on the population parameters and generate these parameters. Then in the second step, we further generate $\{(x_i,y_i)\}$ and $(x,y)$ given the parameters from the first step. The following two assumptions provide the details on how $(x,y)$ (as well as $\{(x_i,y_i)\}$) and the parameters are generated:

\begin{assumption}[Generate $(x,y)$]\label{assumption:data}
    Assume $x\in \mathbb{R}^m$ and $y\in\{-1,+1\}$. Given parameters $(\theta_+,\theta_-,\pi,p_+,p_-)$, to generate $(x,y)$, $y$ is first generated from a Bernoulli distribution with $\pi$, i.e. $P(y=+1)=\pi$ and $P(y=-1)=1-\pi$, then $x$ is generated from a class-wise input distribution accordingly. 
    Given $y=+1$, $x$ follows a Gaussian distribution $N(\theta_+,\sigma_+^2I)$ with probability $p_+$ and is sampled from $N(\theta_-,\sigma_-^2I)$ with probability $1-p_+$; given $y=-1$, $x$ is sampled from a Gaussian distribution $N(\theta_-,\sigma_-^2I)$ with probability $p_-$ and follows $N(\theta_+,\sigma_+^2I)$ with probability $1-p_-$. In addition, the examples are independent with each other.
\end{assumption}
Assumption \ref{assumption:data} follows the standard Gaussian mixture design studied in classification, e.g., \cite{dan2020sharp,taheri2022asymptotic,wang2021benign}. The data generation model is determined by a group of important parameters, $\theta_+,\theta_-$, which define the means of the distributions for positive and negative classes; $\pi$, which denotes the expected proportion of samples from different classes; $p_+$ and $p_-$, which introduce ``label noise'' such that when $y=+1$, the corresponding $x$ can be from either of the two clusters. 

Besides generating $(x,y)$ following Assumption \ref{assumption:data} under a given set of parameters,
we also impose the following Assumption \ref{assumption:param} on how the parameters differ among datasets.

\begin{assumption}[Parameters]\label{assumption:param}
    The parameter distributions for pre-training and the inference stage as as follows:
\begin{itemize}[leftmargin=0.5cm]

    \item Pre-train: $\theta_+\sim N(\theta_M,\sigma^2_{M}I_m), \theta_-\sim N(-\theta_M, \sigma^2_{M}I_m)$; $p_+=p_-=1$; $\pi\sim Beta(1,1)$. 
    \item Examples: $\theta_+\sim N(\theta^e_+, \sigma^2_{e+}I_m), \theta_-\sim N(\theta^e_-, \sigma^2_{e-}I_m)$; $p^e_+, p^e_-\in [0,1]$, $\pi\in [0,1]$.
    \item Test data: $\theta_+\sim N(\theta^e_+, \sigma^2_{e+}I_m), \theta_-\sim N(\theta^e_-, \sigma^2_{e-}I_m)$; $p_+=p_-=1$. 
    \item Examples and the test data in the same prompt share the same realization of $(\theta_+,\theta_-)$.
\end{itemize}
{Following Assumption \ref{assumption:data} and \ref{assumption:param}, the pre-training stage, all $(x_i,y_i)$s and $(x,y)$ in the same prompt are conditionally independent and share the same parameters. At the inference stage, the examples are conditionally independently sampled given the parameters and may incur label noise. For the test data $(x,y)$, while it shares the same $(\theta_+,\theta_-)$ with the examples in the prompt, we do not further consider label noise in the test data. The proportion $P(y=+1)$ is not considered in the test data because the later accuracy analysis is performed on $y=+1$ and $y=-1$ separately. }
\end{assumption}

Assumption \ref{assumption:param} aligns with the common scenarios of ICL, i.e., the pre-training distribution and the example distribution at the inference stage can differ. In pre-training, we take $p_+=p_-=1$ to simplify the derivation. In this case, there is no label noise, and the misclassification of the Bayes classifier is only caused by the overlap of the two Gaussian clusters in the distribution. At the inference stage, the examples may have a distribution shift compared to the pre-training data, and we also consider potential label noise in the examples. 

\subsection{ICL Decision and Prediction Accuracy}\label{sec:decision}

During pre-training, the LLM $M$ is trained to learn the distribution of the prompts following Assumption \ref{assumption:data} and \ref{assumption:param}. At the inference stage, a common understanding is that $M$ will make predictions following the pre-training distribution \citep{xie2021explanation, lin2024dual}. In our scenario, this means that $M$ learns $\theta_M$, $\sigma_M^2$, $\sigma_+^2$, $\sigma_-^2$, and the distributions of $x_i,y_i,\pi$ from the data, and uses them to infer the posterior distribution of $\theta_+$ and $\theta_-$ given the examples $\{(x_i,y_i)\}$ for prediction of $x$.

To compute the ICL prediction accuracy, we first derive the posterior distribution of the parameters $(\theta_+,\theta_-,\pi)$ given the examples $\{(x_i,y_i)\}$ and the pre-training knowledge of $\theta_M$, and then use $(\theta_+,\theta_-,\pi)$ to figure out the ICL accuracy.

\textbf{Posterior of parameters.}~~Our goal is to compute the posterior distribution of $\theta_+,\theta_-,\pi$ given examples $(x_1,y_1),...,(x_k,y_k)$. Recall that in Assumption \ref{assumption:param}, 
$p_+=p_-=1$ in the pre-training stage, and  $p^e_+,p^e_-$ at the inference stage can be in $[0,1]$ if label noise occurs. 
Denote $\#(y_i=+1)$ and $\#(y_i=-1)$ as the number of examples with positive/negative labels respectively. The following lemma presents the posterior distribution of $\pi,\theta_+,\theta_-$.

\begin{lemma}\label{lem:posterior_param}
Under Assumption \ref{assumption:data} and Assumption \ref{assumption:param}, the posterior distribution of $\pi,\theta_+,\theta_-$ satisfies
$$
    P(\pi|\{(x_i,y_i)\}_{i\in[k]},M)\propto \pi^{\#(y_i=+1)}(1-\pi)^{\#(y_i=-1)},
$$
$$\theta_+\sim N\left(\frac{\sigma_+^2\theta_M+\sigma_M^2\sum_{y_i=+1}x_i}{\sigma_+^2+\#(y_i=+1)\sigma^2_M},\frac{\sigma_+^2\sigma_M^2}{\sigma_+^2+\#(y_i=+1)\sigma^2_M}I\right)\triangleq N(\hat{\theta}_+,\sigma^2_{\theta_+}I),$$
and
$$\theta_-\sim N\left(\frac{\sigma_M^2\sum_{y_i=-1}x_i-\sigma_-^2\theta_M}{\sigma_-^2+\#(y_i=-1)\sigma_M^2},\frac{\sigma_M^2\sigma_-^2}{\sigma_-^2+\#(y_i=-1)\sigma_M^2}I\right)\triangleq N(\hat{\theta}_-,\sigma^2_{\theta_-}I).$$
\end{lemma}

The proof of Lemma \ref{lem:posterior_param} can be found in Section \ref{sec:appendix:proof:posterior_param}. In short, since the examples $\{(x_i,y_i)\}$ are given, we can directly write the likelihood for $(\pi,\theta_+,\theta_-)$ to derive the corresponding posterior distributions. The mean of the posterior of $\theta_+$ is the weighted average of the mean of $\theta_+^e$, i.e., the pre-training knowledge, and the average of examples with positive labels, indicating a clear shift from pre-training knowledge to the knowledge in the examples when increasing $k$. 
The same phenomenon happens in the negative class.

\textbf{ICL decision.}~~Given Lemma \ref{lem:posterior_param}, denoting $z_k={(\#(y_i=-1)+1)}/{(\#(y_i=+1)+1)}$, the following lemma shows the ICL decision boundary for the test data:
\begin{lemma}\label{lem:decision}
    Under Assumption \ref{assumption:data} and Assumption \ref{assumption:param}, the probability of $y=+1/-1$ is as follows
\begin{eqnarray*}
    P(y=+1|x,\{(x_i,y_i)\},M)&=&\frac{P(x,y=+1|\{(x_i,y_i)\},M)}{P(x,y+=1|\{(x_i,y_i)\},M)+P(x,y=-1|\{(x_i,y_i)\},M)}\\
    &=&\frac{\frac{\#(y_i=+1)+1}{k+2} N(x)}{\frac{\#(y_i=+1)+1}{k+2} N(x)+\frac{\#(y_i=-1)+1}{k+2}},\\
    P(y=-1|x,\{(x_i,y_i)\},M)&=&\frac{P(x,y=-1|\{(x_i,y_i)\},M)}{P(x,y+=1|\{(x_i,y_i)\},M)+P(x,y=-1|\{(x_i,y_i)\},M)}\\
    &=&\frac{\frac{\#(y_i=-1)+1}{k+2}}{\frac{\#(y_i=+1)+1}{k+2} N(x)+\frac{\#(y_i=-1)+1}{k+2}},
\end{eqnarray*}
where $$N(x)=\left(\sqrt{\frac{\sigma_-^2+\sigma_{\theta_-}^2}{\sigma_+^2+\sigma_{\theta_+}^2}}\right)^m \exp \left[-\frac{(x-\hat{\theta}_+)^\top(x-\hat{\theta}_+)}{2(\sigma^2_++\sigma^2_{\theta_+})}+\frac{(x-\hat{\theta}_-)^\top(x-\hat{\theta}_-)}{2(\sigma^2_-+\sigma^2_{\theta_-})}\right].$$

The decision boundary is $\hat{y}_{ICL}=1(f_{ICL}(x)>0)$, where $f_{ICL}(x)=N(x)-z_k$.
\end{lemma}
The proof of Lemma \ref{lem:decision} can be found in Section \ref{sec:appendix:proof:decision}. When $(\pi,\theta_+,\theta_-)$ are fixed, given $y$, $x$ follows a Gaussian distribution. When integrating over all possible $(\pi,\theta_+,\theta_-)$, the marginal distribution of $x$ given $y$ still follows a Gaussian distribution. Hence, $(x,y)$ marginally follows a Gaussian mixture distribution, and the decision boundary can be further obtained.
From Lemma \ref{lem:decision}, we can see how the pre-training distribution ($\theta_M,\sigma_M^2$) and examples $\{(x_i,y_i)\}$ impact $P(y=+1|x,\{(x_i,y_i)\},M)$ and $P(y=-1|x,\{(x_i,y_i)\},M)$, further changing the decision boundary correspondingly. 
The pre-training knowledge $(\theta_M,\sigma_M^2)$ and examples $\{(x_i,y_i)\}$ first determines $(\hat{\theta}_+, \hat{\theta}_-,\sigma_{\theta_+}^2,\sigma_{\theta_-}^2)$, the latter of which then determines the decision boundary.
More details about the interplay of pre-training and examples under different scenarios will be provided in Section \ref{sec:demonstration}.

Furthermore, the conditional probability of predicting positive or negative labels is closely related to $z_k$, which represents the label fraction in the examples. A smaller $z_k$ indicates a greater number of positive labels in the examples, i.e., a larger value of $\#(y_i=+1)$, and may suggest a higher likelihood of predicting a positive label. This imbalanced scenario is formally addressed in Proposition \ref{prop:imbalance} in Section \ref{sec:demonstration}.

\textbf{ICL Accuracy.}~~After obtaining the decision boundary from Lemma \ref{lem:decision}, we finally provide the general formula of the ICL prediction accuracy. 
The following Theorem \ref{thm:icl_acc} derives the exact ICL accuracy.
It is worth noting that though the decision boundary of ICL follows the Bayesian inference procedure, we consider fixed $\theta^e_+$ and $\theta_+$ when deriving the classification accuracy.

\begin{theorem}\label{thm:icl_acc}
Under Assumption \ref{assumption:data} and Assumption \ref{assumption:param}, further assume $\sigma_+^2=\sigma_-^2=\sigma^2$, then we have the following probability of correct prediction for each class. Assume \blue{$\pi\gg \log k/\sqrt{k}$ and $1-\pi\gg \log k/\sqrt{k}$}. For fixed $\theta^e_+$, $\theta^e_-$, $p_+\in(0,1)$, $p_-\in(0,1)$ at the inference stage, denote $\Theta\triangleq (\pi,\theta^e_+,\theta^e_-,p_+,p_-)$, and
$$z_k\triangleq \log\left(\frac{\#(y_i=-1)+1}{\#(y_i=+1)+1}\right), m_k\triangleq\sigma^2\left(z_k-\frac{m}{2}\log\left(\frac{\sigma^2+\sigma_{\theta_-}^2}{\sigma^2+\sigma_{\theta_+}^2}\right)\right),$$
then \blue{for some positive constant $c_x$}, with probability at least $$1-\Phi\left( -\frac{c_k\log k}{\sqrt{\pi(1-\pi)}} \right)- \exp\left( - {c_x \log^2k{(\pi k-c_k\sqrt{k}\log k)}} \right) +O\left(\frac{1}{k^{3/2}}\right)$$ over the randomness of $\{(x_i,y_i)\}$,
\blue{\begin{eqnarray*}
    P(correct|y=+1, \{(x_i,y_i)\},\Theta,M)&=&
    \underbrace{1-\Phi\left(  \frac{m_k}{\sqrt{(\sigma^2+\sigma^2_{e+})}\|\hat{\theta}_--\hat{\theta}_+\|}-
    \frac{\theta_+^e-(\hat\theta_++\hat\theta_-)/2}{\sqrt{(\sigma^2+\sigma^2_{e+})}}\frac{\hat{\theta}_--\hat{\theta}_+}{\|\hat{\theta}_--\hat{\theta}_+\|}
    \right)}_{\text{dominant term}}\\
    &&+\underbrace{O\left(\frac{c\log k}{k}\right)+O\left( P\left(\mathcal{X}^2_m >\pi c\log k(\pi k-c_k\sqrt{k}\log k) \right) \right)}_{\text{remainder terms}},
    \end{eqnarray*}}
    and with probability at least
    $$1-\Phi\left( -\frac{c_k\log k}{\sqrt{\pi(1-\pi)}} \right)-\exp\left( - {c_x \log^2k{((1-\pi) k-c_k\sqrt{k}\log k)}} \right) +O\left(\frac{1}{k^{3/2}}\right),$$
    \blue{\begin{eqnarray*}
    P(correct|y=-1, \{(x_i,y_i)\},\Theta,M)&=&
    \underbrace{\Phi\left(  \frac{m_k}{\sqrt{(\sigma^2+\sigma^2_{e+})}\|\hat{\theta}_--\hat{\theta}_+\|}-
    \frac{\theta_-^e-(\hat\theta_++\hat\theta_-)/2}{\sqrt{(\sigma^2+\sigma^2_{e+})}}\frac{\hat{\theta}_--\hat{\theta}_+}{\|\hat{\theta}_--\hat{\theta}_+\|}
    \right)}_{\text{dominant term}},\\
    &&+\underbrace{O\left(\frac{c\log k}{k}\right)+O\left( P\left(\mathcal{X}^2_m >(1-\pi) c\log k((1-\pi) k-c_k\sqrt{k}\log k) \right) \right)}_{\text{remainder terms}},
\end{eqnarray*}}
where $\Phi(\cdot)$ is the cumulative distribution function of a standard Gaussian distribution, and $\mathcal{X}^2_m$ denotes a random variable following Chi-square distribution with degree of freedom $m$. 
\end{theorem}

To prove Theorem \ref{thm:icl_acc}, we first obtain the marginal distribution of $x|y$ given the example distribution in Assumption \ref{assumption:param} to remove the internal parameters from the formulation, denoted as $P(x|y=+1)$. Then the ICL performance, i.e., prediction accuracy, can be computed via $P(correct|y=+1, \{(x_i,y_i)\},M)=\int_{f_{ICL}(x)\ge 0}P(x|y=+1)dx$. Similar steps apply for $y=-1$. \blue{The conditions $\pi\gg \log k/\sqrt{k}$ and $1-\pi\gg \log k/\sqrt{k}$ensure that there are enough examples for each class so that the remainder terms are negligible, otherwise $\#(y_i=+1)/k$ or $\#(y_i=-1)/k$ will not converge to $\pi$ in a rate of $O_p(1/\sqrt{k})$. In terms of $c_k$, on one hand, it is used to quantify the tail probability when $|\#(y_i=+1)/k-\pi|$ or $|\#(y_i=-1)/k-(1-\pi)|$ exceeds $(c_k\log k)/\sqrt{k}$ as in (\ref{eqn:k_11}) and (\ref{eqn:k_12}) in Section \ref{sec:appendix:proof:icl_acc}. On the other hand, this tail event is further used to quantify the behavior of $\sigma_{\theta_+}^2$ and $\sigma_{\theta_-}^2$ as in (\ref{eqn:k_21}) and (\ref{eqn:k_22}) in Section \ref{sec:appendix:proof:icl_acc}. So $c_k$ appears in two terms in the tail probability and affects the probability in two different directions. } The detailed proof of Theorem \ref{thm:icl_acc} can be found in Section \ref{sec:appendix:proof:icl_acc}.

Theorem \ref{thm:icl_acc} decomposes the accuracy of positive and negative classes into a dominant probability term and small negligible terms. \blue{A detailed comparison on the order of the terms is postponed to Proposition \ref{prop:asymp} after we explicitly figure out the rate of the dominant term.}
Theorem \ref{thm:icl_acc} also describes how the accuracy is affected by the interplay between the pre-training knowledge and the examples. Take $P(correct|y=+, \{(x_i,y_i)\},M)$ as an example. When fixing the parameters of the test sample, this probability is determined by ${\hat{\theta}_--\hat{\theta}_+}$ and $\hat{\theta}_-+\hat{\theta}_+$, both of which are mixtures of examples and pre-training knowledge based on Lemma \ref{lem:posterior_param}. The example size $k$, the variance of data $\sigma^2$, as well as pre-training distribution $\sigma_M^2$ will also affect ${\hat{\theta}_--\hat{\theta}_+}$ and $\hat{\theta}_-+\hat{\theta}_+$. Some detailed examples of how the distribution of the examples affects the ICL performance are in the following section.

After obtaining the prediction accuracy given $\#(y_i=+1)$ and $(\hat\theta_+,\hat\theta_-)$, one can further figure out the Bahadur representation of the accuracy, which can be used to derive the asymptotic convergence of the estimates.

Due to possible label corruption in the demonstration, $\hat\theta_+$ and $\hat\theta_-$ may not converge to $\theta_+^e$ and $\theta_-^e$. Therefore, we denote $\tilde\theta_+$ and $\tilde\theta_-$ to represent their expectations over examples $(x_i,y_i)$ respectively. Intuitively, if $k$ increases, $\hat\theta_+$ and $\hat\theta_-$ converges to $\tilde\theta_+$ and $\tilde\theta_-$ respectively, and $\#(y_i=+1)$ also converges to the underlying $\pi$ in the prompt. Therefore, we replace $\hat\theta_+,\hat\theta_-,\#(y_i=+1)$ in Theorem \ref{thm:icl_acc} and denote
\begin{eqnarray*}
    &&P_+^*\triangleq 
    1-\Phi\left(  \frac{\sigma^2\log ((1-\pi)/\pi)}{\sqrt{(\sigma^2+\sigma^2_{e+})}\|\tilde\theta_--\tilde\theta_+\|}-
    \frac{(\theta_+^e-(\tilde\theta_++\tilde\theta_-)/2)^\top(\tilde\theta_+-\tilde\theta_-)}{\|\tilde\theta_+-\tilde\theta_-\|\sqrt{(\sigma^2+\sigma^2_{e+})}} 
    \right),
\end{eqnarray*}
and
\begin{eqnarray*}
    &&P_-^*\triangleq 
    \Phi\left(  \frac{\sigma^2\log ((1-\pi)/\pi)}{\sqrt{(\sigma^2+\sigma^2_{e-})}\|\tilde\theta_--\tilde\theta_+\|}-
  \frac{(\theta_-^e-(\tilde\theta_++\tilde\theta_-)/2)^\top(\tilde\theta_+-\tilde\theta_-)}{\|\tilde\theta_+-\tilde\theta_-\|\sqrt{(\sigma^2+\sigma^2_{e-})}} 
    \right).
\end{eqnarray*}
The following proposition demonstrates the Bahadur representation of $P(correct|y=+1, \{(x_i,y_i)\},M)-P_+^*$. The result for $P(correct|y=-1, \{(x_i,y_i)\},M)-P_-^*$ is similar.
\begin{proposition}[Bahadur Representation]\label{prop:asymp}
Under the conditions in Theorem \ref{thm:icl_acc}, with the same probability as Theorem \ref{thm:icl_acc},
\blue{\begin{eqnarray*}
    &&P(correct|y=+1, \{(x_i,y_i)\},\Theta,M)-P_+^*\\
    &=&\underbrace{-\phi\left(  \frac{\sigma^2\log ((1-\pi )/\pi )}{\sqrt{(\sigma^2+\sigma^2_{e+})}\|\tilde\theta_--\tilde\theta_+\|}-
  \frac{(\theta_+^e-(\tilde\theta_++\tilde\theta_-)/2)^\top(\tilde\theta_+-\tilde\theta_-)}{\|\tilde\theta_+-\tilde\theta_-\|\sqrt{(\sigma^2+\sigma^2_{e+})}}  
    \right)(B_1-B_2)}_{:=R_1}\\
    &&+\underbrace{O\left(\phi'\left(  \frac{\sigma^2\log((1-\pi )/\pi )}{\sqrt{(\sigma^2+\sigma^2_{e+})}\|\theta_-^e-\theta_+^e\|}-
   \frac{(\theta_+^e-(\tilde\theta_++\tilde\theta_-)/2)^\top(\tilde\theta_+-\tilde\theta_-)}{\|\tilde\theta_+-\tilde\theta_-\|\sqrt{(\sigma^2+\sigma^2_{e+})}} \right)(B_1-B_2)^2
    \right)}_{:=R_2}\\
    &&+\underbrace{O\left(\frac{c\log k}{k}\right)}_{:=R_3}+\underbrace{O\left( P\left(\mathcal{X}^2_m >\frac{c\log k}{k}(\pi k-c_k\sqrt{k}\log k) \right) \right)}_{:=R_4}.
\end{eqnarray*}}
    The Bahadur representation of $B_1$ and $B_2$ can be found in Appendix \ref{sec:appendix:proof:bahadur} Equation (\ref{eqn:B1}) and (\ref{eqn:B2}). In addition, $\sqrt{k}B_1$ and $\sqrt{k}B_2$ asymptotically converge to some Gaussian distributions respectively, in which the means asymptotically converge to zero, and the variances are some constants respectively. The notation $\phi$ is the density function of the standard Gaussian distribution, and $\phi'$ is the derivative of $\phi$. 

    \blue{To compare the order of $R_1$ to $R_4$, $R_1=O_p(1/\sqrt{k})$, $R_2=O_p(1/k)$, $R_3=O((\log k)/k)$ if $c$ is a constant, and $R_4$ is an exponential tail and is much smaller than $R_1$ to $R_3$.}
\end{proposition}
The proof of Proposition \ref{prop:asymp} is postponed to the Appendix \ref{sec:appendix:proof:bahadur}. Extending from Theorem \ref{thm:icl_acc}, Proposition \ref{prop:asymp} further derives the Bahadur representation of $m_k$, $\hat\theta_+$ and $\hat\theta_-$, and aggregate them together.
Proposition \ref{prop:asymp} shows how $P(correct|y=+1, \{(x_i,y_i)\},M)$ converges to $P_+^*$: the dominant terms in the difference between $P(correct|y=+1, \{(x_i,y_i)\},M)$ and $P_+^*$ is caused by $B_1$ and $B_2$, and asymptotically, the convergence rate is $O(1/\sqrt{k})$.

\begin{remark}
    \blue{Under Theorem  \ref{thm:icl_acc} and Proposition \ref{prop:asymp}, since $\sigma_+^2=\sigma_-^2$, the decision boundary is a hyperplane, which simplifies the integration in Lemma \ref{thm:icl_acc}. If $\sigma_+^2\neq\sigma_-^2$, the decision boundary will no longer be a hyperplane. We provide technical discussions in Appendix \ref{sec:appendix:noncentral}. In this case, the boundary is a sphere, and one needs to use a noncentral Chi-square distribution to figure out $P(correct|y=+1)$ and $P(correct|y=-1)$. }
\end{remark}

\subsection{Different Demonstration Scenarios}\label{sec:demonstration}

In the following, we extend the above results to investigate how ICL is affected in specific situations. In real practice, it is often observed that the real data contradicts the knowledge from the pre-training data, or the quality is not perfect \citep{carlini2024poisoning, he2024data}. Therefore, we consider contradicting knowledge, imbalanced examples, and label noise in the following.

\textbf{Contradicting knowledge.}~~To study this case, we compare $\theta_+^e=-\theta_M=-\theta_-^e$ and $\theta_+^e=\theta_M=-\theta_-^e$, i.e., the input distribution in examples is the opposite/same to that of pre-training distribution. The following result is obtained based on Theorem \ref{thm:icl_acc} in these scenarios:
\begin{proposition}[Contradicting knowledge]\label{prop:contradict}
    Assume the conditions of Theorem \ref{thm:icl_acc} hold, and also assume $\sigma_{e_+}^2,\sigma_{e_-}^2\rightarrow 0$, and $\pi=0.5$ at the inference stage. Then when $k\sigma_M^2\ll\sigma^2$, i.e., insufficient example size,
    \begin{eqnarray*}
        &&P(correct|y=+1,\theta_+^e=\theta_M=-\theta_-^e,\Theta)-P(correct|y=+1,\theta_+^e=-\theta_M=-\theta_-^e,\Theta)\\
        &\rightarrow&\Phi\left(\frac{\|\theta_M\|}{\sqrt{\sigma^2}}\right)-\Phi\left(-\frac{\|\theta_M\|}{\sqrt{\sigma^2}}\right)>0.
    \end{eqnarray*}
    When $k\sigma_M^2\gg\sigma^2$, i.e., sufficient example size, both $P(correct|y=+1,\theta_+^e=-\theta_M=-\theta_-^e)$ and $P(correct|y=+1,\theta_+^e=\theta_M=-\theta_-^e,\Theta)$ converges to $1-\Phi\left(-{\|\theta_M\|}/{\sqrt{\sigma^2}}\right)$. The accuracy of $y=-1$ exhibits a similar behavior.
\end{proposition}

The proof of Proposition \ref{prop:contradict} can be found in Section \ref{sec:appendix:proof:contradict}. We mainly follow the result in Theorem \ref{thm:icl_acc} and calculate the probabilities under the specific scenario. 

There are two observations in Proposition \ref{prop:contradict}. First, when there are no enough examples and the pre-training knowledge contradicts to the knowledge in the examples and the test data, there is an obvious drop in ICL performance compared to the case when the knowledge matches. Second, when there are enough examples, the knowledge from the examples will dominate, and ICL performance of contradicting knowledge converges to that of matching knowledge. Practically, in real-world LLM applications such as retrieval-augmented generation (RAG)\citep{gao2023retrieval} and LLM-based agents \citep{xi2023rise}, pre-trained LLMs might not be perfectly aligned with the external knowledge and may contain diverse information. Proposition \ref{prop:contradict} implies that with an adequate number of examples, the potential reliance from pre-training knowledge can be minimized, leading to improved performance.



\textbf{Imbalanced examples.}~~In the following, we consider the case where the two classes are imbalanced at the inference stage, i.e. $\pi\rightarrow 0$ or $\pi\rightarrow 1$. In this case, the value of $\pi$ will impact the ICL prediction.

\begin{proposition}[Imbalanced examples]\label{prop:imbalance} 
    Under Assumption \ref{assumption:data} and Assumption \ref{assumption:param}, assume $\sigma^2$ and $\sigma_M^2$ are constants,  $\pi k\rightarrow \infty$ and $(1-\pi)k\rightarrow \infty$, then when $\pi\rightarrow 0$, $P(correct|y=+1, \Theta, \{(x_i,y_i)\},M)\rightarrow 0$. When $\pi\rightarrow 1$, $P(correct|y=-1, \Theta, \{(x_i,y_i)\},M)\rightarrow 0$.
\end{proposition}
The proof of Proposition \ref{prop:imbalance} is in Section \ref{sec:appendix:proof:imbalance}. In short, we follow Lemma \ref{lem:decision} to obtain the decision boundary. 
Then, we repeat the steps of Theorem \ref{thm:icl_acc} to obtain the conclusion.

Based on Proposition \ref{prop:imbalance}, when the examples for one class are much fewer than the other class, ICL performance for the minor class will significantly drop. This suggests that in real practice, one may consider balancing the positive and negative examples when selecting the in-context examples. In practical applications where one class may naturally occur less frequently than the other, such as fraud detection \citep{abdallah2016fraud}, practitioners can mitigate the risk of underperformance for the minority class via strategically balancing the example set, achieving more reliable outcomes.

\textbf{Label noise.}~~It is common that there exist label noises in the examples for ICL. For example, an example $x_i$ sampled from $N(\theta_-^e,\sigma^2_{e-})$ may be labeled as $+1$. Therefore, we change the value of $p_+^e$ and $p_-^e$ in the examples to see how these changes affect the ICL performance, the result of which is summarized as follows:
\begin{proposition}[Label noise]\label{prop:noise}
    Under the conditions of Theorem \ref{thm:icl_acc}, assume $\sigma^2$ and $\sigma_M^2$ are constants, $\theta_M=\theta_+^e$ and $\theta_M=-\theta_-^e$, and $\sigma_{e+}^2,\sigma_{e-}^2\rightarrow 0$. Also assume $\pi=0.5$ and $k\rightarrow\infty$ at the inference stage. Then
    \begin{itemize}
        \item When $1-p_+^e-p_-^e<0$, $P(correct|y=+1, p_+^e,p_-^e)$ increases in $p_+^e$, and $P(correct|y=-1, p_+^e,p_-^e)$ increases in $p_-^e$. 
        \item When $1-p_+^e-p_-^e=0$, the decision boundary set $\{f_{ICL}(x)>0\}$ will degenerate to either $\emptyset$ or full space. The (positive accuracy, negative accuracy) is either (0,1), (1,0), or (0.5,0.5).
        \item When $1-p_+^e-p_-^e>0$, how $P(correct|y=+1, p_+^e,p_-^e)$ and $P(correct|y=-1, p_+^e,p_-^e)$ change depends on the specific $(p_+^e,p_-^e)$.
    \end{itemize}
\end{proposition}
The proof of Proposition \ref{prop:noise} can be found in Section \ref{sec:appendix:proof:noise} and is a direct extension from Theorem \ref{thm:icl_acc}.

In Proposition \ref{prop:noise}, recall that $p_+^e=p_-^e=1$ implies no random flip on the example labels. Intuitively, when keeping $p_-^e=1$ and decreasing $p_+^e$, the positive class becomes a mixture of two Gaussian distributions. In this case, $\hat\theta_+$ is closer to zero, and the decision boundary will shift towards $-\theta_M$. Therefore, it is more likely that ICL predicts a negative label for $x$, which aligns with the change in $P(correct|y=+1, p_+^e,p_-^e)$ and $P(correct|y=-1, p_+^e,p_-^e)$ in Proposition \ref{prop:noise}. On the other hand, when $1-p_+^e-p_-^e\geq 0$, the large label noise will jeopardize the accuracy of $\hat\theta_+$ and $\hat\theta_-$, and it is hard to tell the exact prediction accuracy. To summarize, Proposition \ref{prop:noise} provides a quantification of the effect of label noise and characterizes the tendency of ICL predictions relative to the noise level, measured by $p^e_+$ and $p^e_-$. This helps clarify how varying levels of noise in example labels can affect the accuracy and reliability of ICL, highlighting the importance of high-quality examples when applying ICL in practical scenarios.
\section{Experiments}\label{sec:experiment}

In this section, we empirically verify the analysis in Section \ref{sec:analysis independent}. In summary, both simulation and real-data experiments are consistent Section \ref{sec:analysis independent}\footnote{Code is available in \url{https://anonymous.4open.science/r/ICL-understanding-classification-DC1C}}.

\subsection{Simulation}\label{exp:simulation}

To set up the experiment, we pre-train a decoder-only Transformer \citep{vaswani2017attention} from the GPT-2 \citep{radford2019language} family. We follow Section \ref{sec:assumption} to construct the pre-training data and follow \citep{garg2022can} to perform next-token prediction to estimate all $y_i$s and $y_{query}$. During the pre-training, we sample a new pair of $(\theta_+,\theta_-)$ for each iteration and generate corresponding demonstration examples. A detailed setting for simulation can be found in Appendix \ref{sec:appendix:simulation setup}. 
 
 We implement the scenarios as in Section \ref{sec:demonstration} as follows:

\begin{figure}[t]
\centering
\hfill
\begin{minipage}{0.35\textwidth}
    \includegraphics[width=\linewidth]{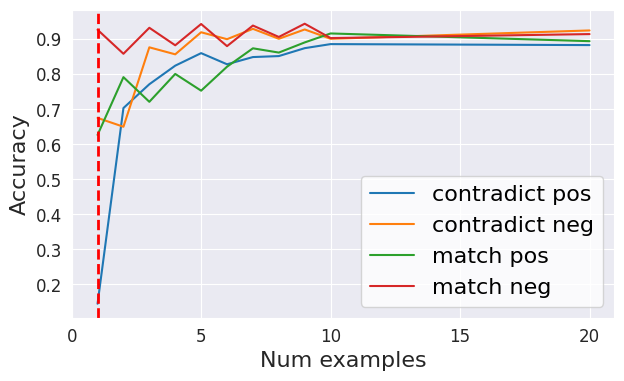} 
        \caption{Contradict knowledge, $\sigma^2=1$}
        \label{fig:contradict1}
\end{minipage}\hfill
\begin{minipage}{0.35\textwidth}
    \includegraphics[width=\linewidth]{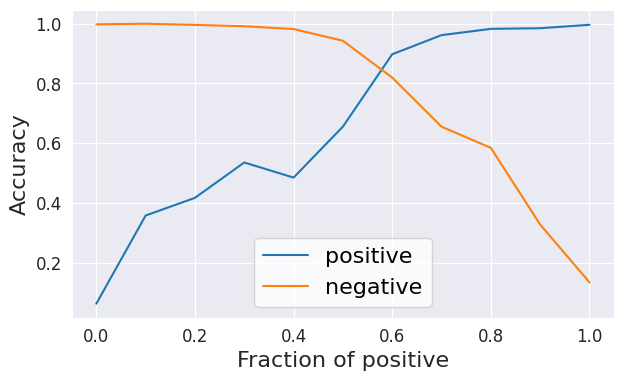}
    \caption{Imbalanced examples}
    \label{fig:imbalance}
\end{minipage}\hfill

\end{figure}

\textbf{Contradicting knowledge.}~~We pre-train the model with $\theta_M=0.5 \textbf{1}_5, \sigma^2_M=1, \sigma^2=1$. During the inference stage, we generate examples and test data with $\theta^e_+=-0.5\textbf{1}_5, \theta^e_-=0.5\textbf{1}_5$ which contradicts the pre-training distribution. We let $\sigma^2_{e+}=\sigma^2_{e-}=1$, and test on various $\sigma^2\in\{1, 2, 4\}$.
The results for $\sigma^2= 2, 4$ are postponed to Section \ref{sec:appendix:experiments:independent_simulation}. We compute the ICL accuracy for each class when $k$ increases. For comparison, we also generate examples with matched knowledge ($\theta^e_+=0.5\textbf{1}_5=-\theta^e_-$) and examine the ICL accuracy. 

The results can be found in {\color{blue}Figure \ref{fig:contradict1}}, where the X-axis represents the number of examples $k$, and the Y-axis is the ICL accuracy. The red dash denotes $\sigma^2/\sigma^2_M$ based on Proposition \ref{prop:contradict} and $\sigma^2/\sigma^2_M=1$ in our simulation. There are two observations. First, when $k\leq\sigma^2/\sigma^2_M$, the ICL performance of contradicting knowledge is worse than that of matching knowledge, verifying that the transformer heavily relies on the pre-training knowledge when there are limited examples. Second, when $k$ increases, the ICL performance for contradicting knowledge increases to around 87\% when $k=20$, indicating that the knowledge from the examples will dominate when $k$ is large.

\textbf{Imbalanced examples.}~~We pre-train the model with $\theta_M=0.5\textbf{1}_5, \sigma^2_M=1$, $\sigma^2=1$. During the inference stage, we generate examples and the test data with $\theta^e_+=0.5\textbf{1}_5, \theta^e_-=-0.5\textbf{1}_5$ and $\sigma^2_{e+}=\sigma^2_{e-}=1, \sigma^2=1$. We test with various fraction $\pi\in \{0,0.1,0.2,...,0.9,1.0\}$. In Figure \ref{fig:imbalance}, the X-axis represents the fraction of positive examples among all examples in the demonstration, i.e., $\pi$. We can observe that when the fraction of positive is increasing, ICL accuracy for positive inputs is increasing and finally reaches 100\%, while ICL accuracy for negative inputs is decreasing, which is the same as described in Proposition \ref{prop:imbalance}.

\textbf{Label noise.}~~We pre-train the model with $\theta_M=0.5\textbf{1}_5, \sigma^2_M=1$, $\sigma^2=1$. During the inference stage, we generate examples and the test data with $\theta^e_+=0.5\textbf{1}_5, \theta^e_-=-0.5\textbf{1}_5$, and fix $\sigma^2_{e+}=\sigma^2_{e-}=1, \sigma^2=0.01$. The example size $k$ is 100, and $\pi$ is 0.5. For the examples in each class, we randomly flip their label with probability $1-p_+^e, 1-p_-^e$ respectively, and test ICL accuracy for each class for $1-p_+^e,1-p_-^e \in \{1,0.9,0.8,...,0.1,0\}$. The ICL accuracies for each class and the overall result are summarized in \blue{Figure} \ref{fig:label_noise_sim}.

The phenomenon in these heatmaps is consistent with our conclusion in Proposition \ref{prop:noise}. Take the ICL accuracy of the positive class as an example (the left panel of Figure \ref{fig:label_noise_sim}); we observe that when the flipping probability in the negative class is fixed, smaller flipping probability (higher $p_+^e$) in the positive class usually leads to higher accuracy in the positive class. Moreover, the diagonal from the bottom left to the upper right represents cases when $p_+^e+p_-^e=1$ and it is obvious that the positive accuracy is either approximately 0,  0.5, or 1. Similar observations can be found for the negative class as well (the middle panel of Figure\ref{fig:label_noise_sim}). In terms of the overall accuracy, only when both $p_+^e$ and $p_-^e$ are close to 1, the overall accuracy is greater than 80\%.

\begin{figure*}[t]
\centering
\includegraphics[width=0.8\textwidth]{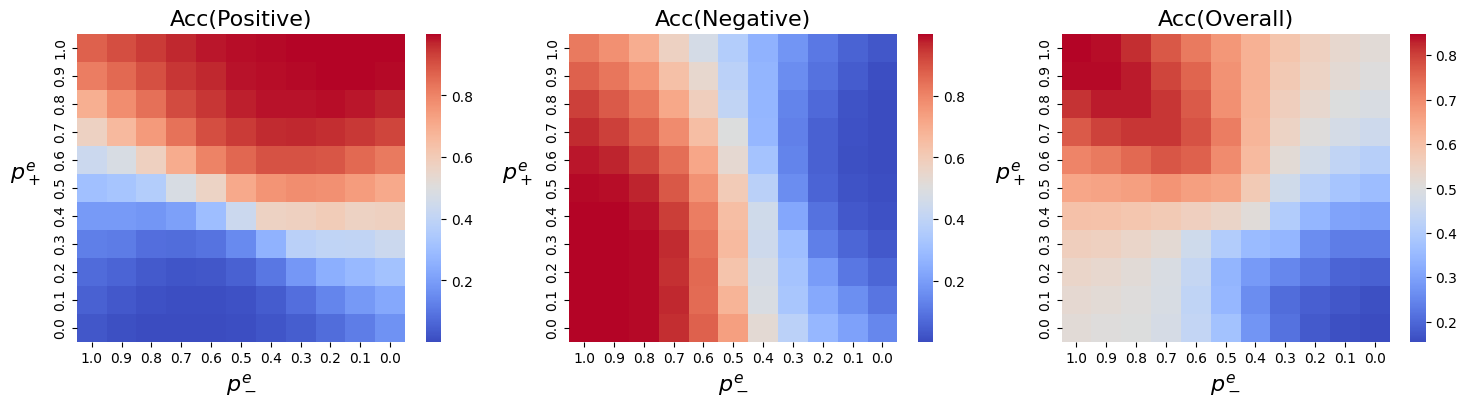}

\caption{Simulation results on positive and negative accuracy facing label noises.}
\label{fig:label_noise_sim}

\end{figure*}

\subsection{Real-Data Experiment}\label{exp:real data}

In this subsection, we conduct experiments on real datasets to show that the theoretical insights in Section \ref{sec:analysis independent} also align with the practical scenarios.

We consider two popular pre-trained LLMs, Pythia-6.9B \citep{biderman2023pythia} and Llama2-7B \citep{touvron2023llama}. We test on a sentiment analysis dataset, SST2 dataset \citep{wang2018glue}, which is also a binary classification task (labeled as ``positive'' and ``negative''). During the inference, we randomly select $k$ samples from the training set as examples and compute the ICL accuracy for each class. We repeat the process 10 times and record the average accuracy. If not specified, $k=50$.

\textbf{Imbalanced examples.}~~We conduct experiments when the fraction $\pi$ of examples from the positive class is not 0.5. Specifically, we test with $\pi\in\{0,0.1,0.2,0.3,0.4,0.5,0.6,0.7,0.8,0.9,1.0\}$. As depicted in Figure \ref{fig:imablance_pythia} and \ref{fig:imablance_llama}, when the number of positive examples increases, the accuracy of the positive class increases as that of the negative class decreases, which also supports our analysis in Proposition \ref{prop:imbalance}.
\begin{figure}[h]
    \centering
    \hfill
    \begin{minipage}{0.35\textwidth}
    \includegraphics[width=\textwidth]{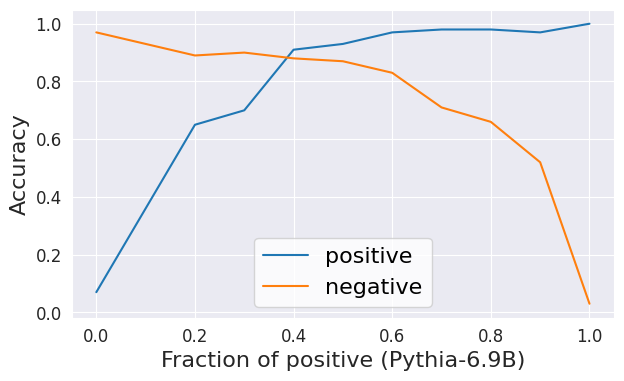}

    \caption{Imbalanced case, Pythia-6.9B.}
    \label{fig:imablance_pythia}
    \end{minipage}
    \hfill
    \begin{minipage}{0.35\textwidth}
    \includegraphics[width=\textwidth]{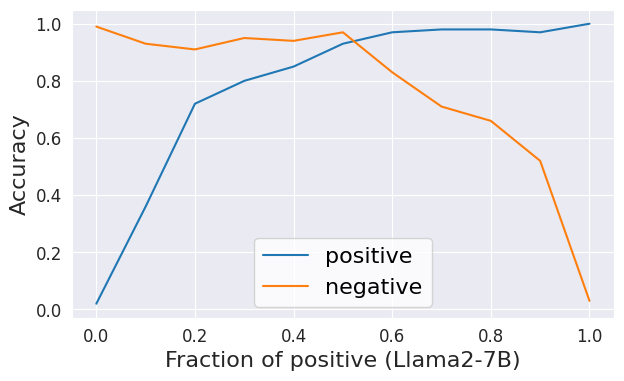}

    \caption{Imbalanced case, Llama2-7B.}
    \label{fig:imablance_llama}
    \end{minipage}\hfill

\end{figure}

\textbf{Label noise.}~~Similar to the simulation, we randomly flip the label of examples from positive and negative classes and the correct probability is $p_+^e,p_-^e$ respectively. We let $p^e_+,p^e_-\in \{0.0,0.1,...,0.9,1.0\}$ and record ICL accuracy in each class. Results are shown in Figure \ref{fig:pythia pos}, \ref{fig:pythia neg}, \ref{fig:llama pos}, \ref{fig:llama neg}. It can be consistently observed that when $p_-^e$ is fixed, larger $p_+^e$ leads to higher accuracy in the positive class (Figure \ref{fig:pythia pos} and \ref{fig:llama pos}); when $p_+^e$ is fixed, larger $p_-^e$ leads to higher accuracy in the negative class (Figure \ref{fig:pythia neg} and \ref{fig:llama neg}). This observation is consistent with our analysis in Proposition \ref{prop:noise}. 
\begin{figure}[!h]
    \centering
    \begin{minipage}{0.24\textwidth}
        \includegraphics[width=\linewidth]{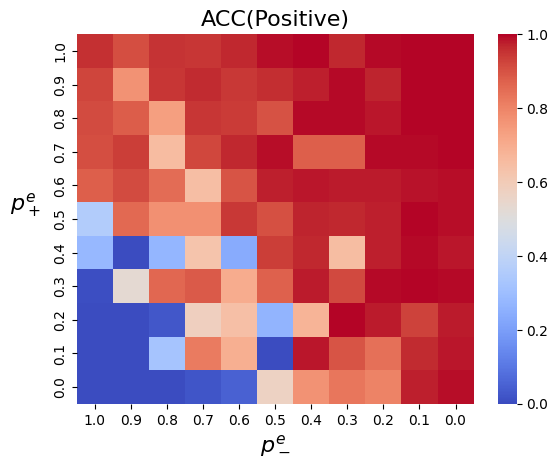}
        \caption{Pythia-6.9B: ICL acc, positive class.}
        \label{fig:pythia pos}
    \end{minipage}\hfill
    \begin{minipage}{0.24\textwidth}
        \includegraphics[width=\linewidth]{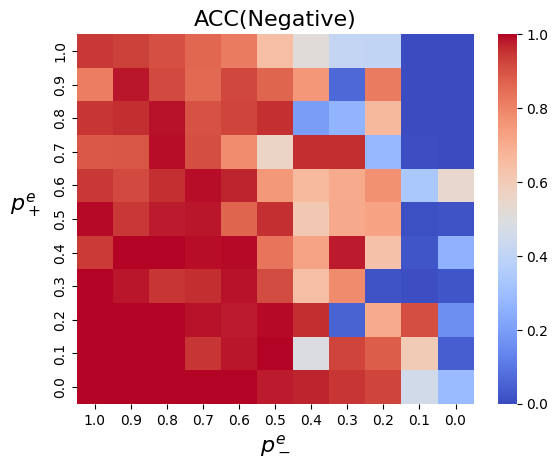}
        \caption{Pythia-6.9B: ICL acc, negative class.}
        \label{fig:pythia neg}
    \end{minipage}
\hfill
    \begin{minipage}{0.24\textwidth}
        \includegraphics[width=\linewidth]{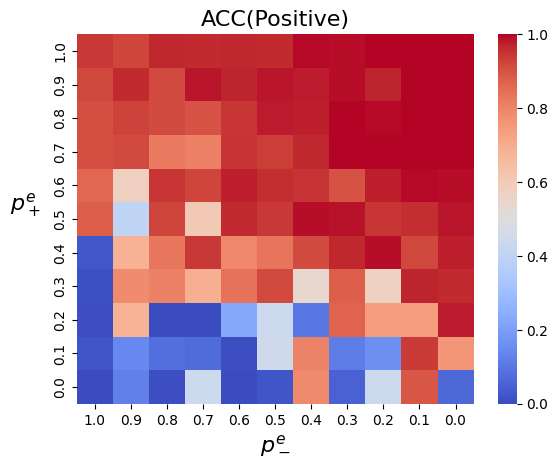}
        \caption{Llama2-7B: ICL acc, positive class.}
        \label{fig:llama pos}
    \end{minipage}
    \hfill
    \begin{minipage}{0.24\textwidth}
        \includegraphics[width=\linewidth]{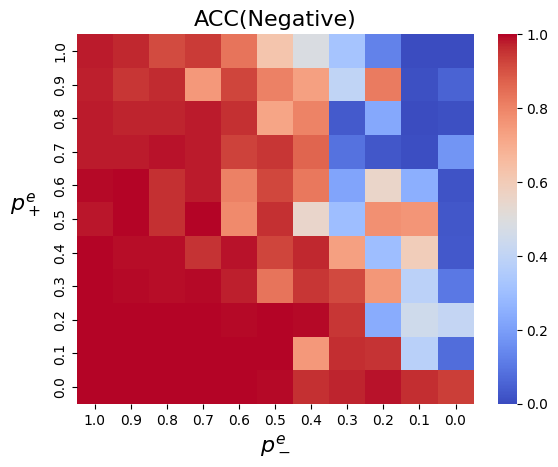}
        \caption{Llama2-7B: ICL acc, negative class.}
        \label{fig:llama neg}
    \end{minipage}
\end{figure}

\section{Mean Reversion in ICL with Dependent Examples} \label{sec:dependent}

The previous analysis and experiments provide a comprehensive understanding of the effect of pre-training and examples under the independent-example scenario. However, it is also common when examples are sampled dependently, especially when examples are strategically selected to serve a specific objective, such as ensuring a balanced representation of 50\% positive and 50\% negative examples to prevent dataset imbalance \citep{wang2024large, zhang2022active}. Surprisingly, we discover a counter-intuitive phenomenon, named as ``\textbf{mean reversion}'', under this scenario. We first empirically illustrate this phenomenon and then provide a theoretical analysis.

\textbf{Empirical illustration of Mean Reversion.}~~We follow a similar procedure as introduced in Section \ref{sec:experiment}, while the difference is that during the pre-training stage, we fix the fraction of positive labels (examples + test data) to be exactly 0.5. This differs from the independent case since the fraction may fluctuate around 0.5 instead of strictly equal to 0.5 when taking $\pi=0.5$ but the examples are i.i.d. generated. During the inference, in-context examples are sampled from both classes but are all labeled as positive or negative. We test with various fractions of positive examples in the prompt to see how the prediction for the test input $x$ is affected. Results are shown in Figure \ref{fig:dep all pos} \ref{fig:dep all neg}. 

In Figure \ref{fig:dep all pos}, all examples are labeled as positive and the X-axis reflects the fraction of examples truly from the positive class; while in \ref{fig:dep all neg}, all examples are labeled as negative and the X-axis reflects the fraction of examples truly from the negative class. There are four lines in the two figures. The ``Positive''/``Negative'' refers to the probability of the prediction being positive/negative when the correct label is positive/negative. The ``Total pos''/``Total neg'' represents the marginal probability of the prediction being positive/negative.

\begin{figure}[h]
    \centering\hfill
    \begin{minipage}{0.35\textwidth}
        \centering
        \includegraphics[width=\linewidth]{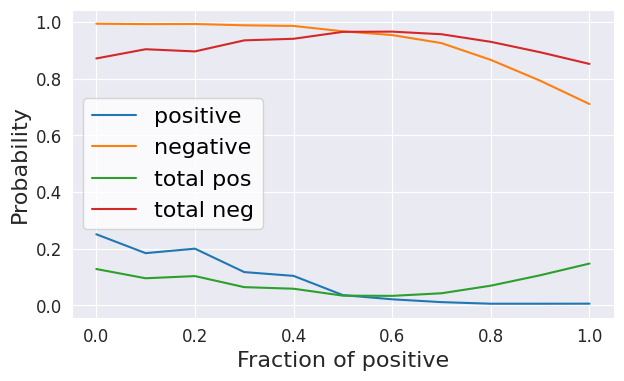}

        \caption{All $y_i$s are positive.}
        \label{fig:dep all pos}
    \end{minipage}\hfill
    \begin{minipage}{0.35\textwidth}
        \centering
        \includegraphics[width=\linewidth]{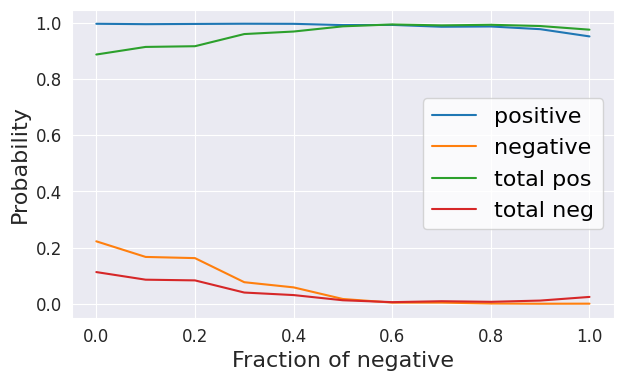} 

        \caption{All $y_i$s are negative.}
        \label{fig:dep all neg}
    \end{minipage}\hfill

\end{figure}

As in Figure \ref{fig:dep all pos}, we can see that regardless of the change in the fraction of examples' true classes (X-axis), 
we always obtain a low positive rate for the true positive test data, and the overall positive rate is low. 
This contradicts the independent case in Figure \ref{fig:imbalance}.
Intuitively, since LLM learns that the fraction of positive labels is exactly 0.5 in the pre-training, at the inference stage, the joint label distribution of examples and test input converges to 0.5-0.5 \footnote{This is similar to the ``mean reversion'' in certain stochastic differential equations (SDEs) where the variable tends to move toward a long-term average over time, thus we also name our observation as ``mean reversion''.}. 

In the following, we first provide the assumption for data distribution and then provide a rigorous theoretical analysis to explain this phenomenon.

\begin{assumption}[Dependent examples]\label{assum:dependent}
    Assume $\{y_i\}$ are not independent. Let $frac$ denote the fraction of $+1$ among the set of labels in the pre-training set. Assume $frac$ approximately\footnote{The fraction number given a finite number of examples follows a discrete distribution. Here we approximate it to the Beta distribution and focus on the intuition. Details can be found in the proof.} follows $\text{Beta}(\alpha,\beta)$ with $\alpha,\beta\rightarrow\infty$ and $\alpha/\beta\rightarrow \pi/(1-\pi)$ 
    for some constant $\pi\in (0,1)$, i.e.,  $frac$ is around $\pi$ with probability tending to 1. Given each $y_i$, the corresponding $x_i$ is sampled independently.
\end{assumption}

The key distinction between Assumption \ref{assum:dependent} and Assumption \ref{assumption:param} is that $frac$ is fixed in the pre-training prompts. For instance, if $frac=0.5$, there will consistently be an equal number of positive and negative examples. This can lead to a fundamentally different outcome, as the language model may learn the fixed fraction $frac$ instead of the true relationship between input $x$ and label $y$. The following theorem formally describes this phenomenon.

\begin{theorem}[Dependent examples]\label{them: dependent}
    Given Assumption \ref{assum:dependent} and further assume that all $y_i$s and $y$ have an equal chance of being positive in pre-training. Then when $P(x|y=+1,\{(x_i,y_i)\},M)$ and $P(x|y=-1,\{(x_i,y_i)\},M)$ are both bounded and bounded away from zero, the following holds:
    \begin{equation*}
    P(y=+1|x,\{(x_i,y_i)\},M)\rightarrow \begin{cases}
        1 &\text{if } frac<\frac{\lfloor \pi (k+1)\rfloor-1}{k+1}\\
        0 &\text{if } frac>\frac{\lceil \pi (k+1)\rceil+1}{k+1}
    \end{cases}.
\end{equation*}

\end{theorem}

We direct the reader into Appendix \ref{app: proof dependent} for the detailed proof. In short, when calculating $P((x,y=+1)|\{(x_i,y_i)\},M)=P(x|y=+1, \{(x_i,y_i)\},M)P(y=+1|\{(x_i,y_i)\},M)$, since the examples are not independent, we need to follow the dependency among $y_i$s and $y$ to determine $P(y=+1|\{(x_i,y_i)\},M)$, which is determined by the relationship between ${\#(y_i=+1)}/{(k+1)}$ in the testing data and $frac$ in pre-training.

Theorem \ref{them: dependent} indicates that the conditional probability of $y$ is determined by the fraction of labels within the pre-training set and the examples during inference, in addition to the inputs. A direct corollary is that when the fraction of $y_i=+1$ is fixed as 0.5 during the pre-training, and all $y_i$s are negative in the inference stage, the prediction for $y$ is always positive. This is consistent with the observation in Figure \ref{fig:dep all pos}, \ref{fig:dep all neg}. While this may be different from how LLMs are pre-trained in real practice, we believe this observation can also help people better understand the mechanism of transformers: They not only learn how to compare examples to perform ICL but are also able to learn the pattern in the prompts.

\section{Conclusion}\label{sec:conclusion}

In this paper, we analyze the behavior of ICL in a binary classification model. We study the ICL performance under different scenarios, including contradicting knowledge, imbalanced examples, and label noise.
In addition to the above analysis in which we assume examples are independently chosen in pre-training, we also find a counter-intuitive phenomenon when the examples are selected in a dependent way. When fixing the number of positive labels and negative labels in the prompt, the ICL prediction behaves in a mean-reversion manner. We believe that our observations and theoretical results can provide deep insights into understanding the behaviors of transformers. A future direction could be to relax the conditions in this paper and consider more general data distributions.

\bibliographystyle{unsrtnat}  
\bibliography{references}  

\begin{thebibliography}{40}
\providecommand{\natexlab}[1]{#1}
\providecommand{\url}[1]{\texttt{#1}}
\expandafter\ifx\csname urlstyle\endcsname\relax
  \providecommand{\doi}[1]{doi: #1}\else
  \providecommand{\doi}{doi: \begingroup \urlstyle{rm}\Url}\fi

\bibitem[Singhal et~al.(2023)Singhal, Azizi, Tu, Mahdavi, Wei, Chung, Scales, Tanwani, Cole-Lewis, Pfohl, et~al.]{singhal2023large}
Karan Singhal, Shekoofeh Azizi, Tao Tu, S~Sara Mahdavi, Jason Wei, Hyung~Won Chung, Nathan Scales, Ajay Tanwani, Heather Cole-Lewis, Stephen Pfohl, et~al.
\newblock Large language models encode clinical knowledge.
\newblock \emph{Nature}, 620\penalty0 (7972):\penalty0 172--180, 2023.

\bibitem[Brown et~al.(2020)Brown, Mann, Ryder, Subbiah, Kaplan, Dhariwal, Neelakantan, Shyam, Sastry, Askell, et~al.]{brown2020language}
Tom Brown, Benjamin Mann, Nick Ryder, Melanie Subbiah, Jared~D Kaplan, Prafulla Dhariwal, Arvind Neelakantan, Pranav Shyam, Girish Sastry, Amanda Askell, et~al.
\newblock Language models are few-shot learners.
\newblock \emph{Advances in neural information processing systems}, 33:\penalty0 1877--1901, 2020.

\bibitem[Garg et~al.(2022)Garg, Tsipras, Liang, and Valiant]{garg2022can}
Shivam Garg, Dimitris Tsipras, Percy~S Liang, and Gregory Valiant.
\newblock What can transformers learn in-context? a case study of simple function classes.
\newblock \emph{Advances in Neural Information Processing Systems}, 35:\penalty0 30583--30598, 2022.

\bibitem[Hendel et~al.(2023)Hendel, Geva, and Globerson]{hendel2023context}
Roee Hendel, Mor Geva, and Amir Globerson.
\newblock In-context learning creates task vectors.
\newblock \emph{arXiv preprint arXiv:2310.15916}, 2023.

\bibitem[Todd et~al.(2023)Todd, Li, Sharma, Mueller, Wallace, and Bau]{todd2023function}
Eric Todd, Millicent~L Li, Arnab~Sen Sharma, Aaron Mueller, Byron~C Wallace, and David Bau.
\newblock Function vectors in large language models.
\newblock \emph{arXiv preprint arXiv:2310.15213}, 2023.

\bibitem[Min et~al.(2022)Min, Lyu, Holtzman, Artetxe, Lewis, Hajishirzi, and Zettlemoyer]{min2022rethinking}
Sewon Min, Xinxi Lyu, Ari Holtzman, Mikel Artetxe, Mike Lewis, Hannaneh Hajishirzi, and Luke Zettlemoyer.
\newblock Rethinking the role of demonstrations: What makes in-context learning work?
\newblock \emph{arXiv preprint arXiv:2202.12837}, 2022.

\bibitem[Liu et~al.(2021)Liu, Shen, Zhang, Dolan, Carin, and Chen]{liu2021makes}
Jiachang Liu, Dinghan Shen, Yizhe Zhang, Bill Dolan, Lawrence Carin, and Weizhu Chen.
\newblock What makes good in-context examples for gpt-$3 $?
\newblock \emph{arXiv preprint arXiv:2101.06804}, 2021.

\bibitem[Lu et~al.(2021)Lu, Bartolo, Moore, Riedel, and Stenetorp]{lu2021fantastically}
Yao Lu, Max Bartolo, Alastair Moore, Sebastian Riedel, and Pontus Stenetorp.
\newblock Fantastically ordered prompts and where to find them: Overcoming few-shot prompt order sensitivity.
\newblock \emph{arXiv preprint arXiv:2104.08786}, 2021.

\bibitem[Wu et~al.(2022)Wu, Wang, Ye, and Kong]{wu2022self}
Zhiyong Wu, Yaoxiang Wang, Jiacheng Ye, and Lingpeng Kong.
\newblock Self-adaptive in-context learning: An information compression perspective for in-context example selection and ordering.
\newblock \emph{arXiv preprint arXiv:2212.10375}, 2022.

\bibitem[Wang et~al.(2024)Wang, Zhu, Saxon, Steyvers, and Wang]{wang2024large}
Xinyi Wang, Wanrong Zhu, Michael Saxon, Mark Steyvers, and William~Yang Wang.
\newblock Large language models are latent variable models: Explaining and finding good demonstrations for in-context learning.
\newblock \emph{Advances in Neural Information Processing Systems}, 36, 2024.

\bibitem[Zhang et~al.(2022)Zhang, Feng, and Tan]{zhang2022active}
Yiming Zhang, Shi Feng, and Chenhao Tan.
\newblock Active example selection for in-context learning.
\newblock \emph{arXiv preprint arXiv:2211.04486}, 2022.

\bibitem[Zhao et~al.(2021)Zhao, Wallace, Feng, Klein, and Singh]{zhao2021calibrate}
Zihao Zhao, Eric Wallace, Shi Feng, Dan Klein, and Sameer Singh.
\newblock Calibrate before use: Improving few-shot performance of language models.
\newblock In \emph{International conference on machine learning}, pages 12697--12706. PMLR, 2021.

\bibitem[Wei et~al.(2022)Wei, Wang, Schuurmans, Bosma, Xia, Chi, Le, Zhou, et~al.]{wei2022chain}
Jason Wei, Xuezhi Wang, Dale Schuurmans, Maarten Bosma, Fei Xia, Ed~Chi, Quoc~V Le, Denny Zhou, et~al.
\newblock Chain-of-thought prompting elicits reasoning in large language models.
\newblock \emph{Advances in neural information processing systems}, 35:\penalty0 24824--24837, 2022.

\bibitem[Zhang et~al.(2023)Zhang, Frei, and Bartlett]{zhang2023trained}
Ruiqi Zhang, Spencer Frei, and Peter~L Bartlett.
\newblock Trained transformers learn linear models in-context.
\newblock \emph{arXiv preprint arXiv:2306.09927}, 2023.

\bibitem[Von~Oswald et~al.(2023)Von~Oswald, Niklasson, Randazzo, Sacramento, Mordvintsev, Zhmoginov, and Vladymyrov]{von2023transformers}
Johannes Von~Oswald, Eyvind Niklasson, Ettore Randazzo, Jo{\~a}o Sacramento, Alexander Mordvintsev, Andrey Zhmoginov, and Max Vladymyrov.
\newblock Transformers learn in-context by gradient descent.
\newblock In \emph{International Conference on Machine Learning}, pages 35151--35174. PMLR, 2023.

\bibitem[Dai et~al.(2022)Dai, Sun, Dong, Hao, Ma, Sui, and Wei]{dai2022can}
Damai Dai, Yutao Sun, Li~Dong, Yaru Hao, Shuming Ma, Zhifang Sui, and Furu Wei.
\newblock Why can gpt learn in-context? language models implicitly perform gradient descent as meta-optimizers.
\newblock \emph{arXiv preprint arXiv:2212.10559}, 2022.

\bibitem[Han et~al.(2023)Han, Wang, Zhao, and Ji]{han2023explaining}
Chi Han, Ziqi Wang, Han Zhao, and Heng Ji.
\newblock Explaining emergent in-context learning as kernel regression.
\newblock 2023.

\bibitem[Aky{\"u}rek et~al.(2022)Aky{\"u}rek, Schuurmans, Andreas, Ma, and Zhou]{akyurek2022learning}
Ekin Aky{\"u}rek, Dale Schuurmans, Jacob Andreas, Tengyu Ma, and Denny Zhou.
\newblock What learning algorithm is in-context learning? investigations with linear models.
\newblock \emph{arXiv preprint arXiv:2211.15661}, 2022.

\bibitem[Ahn et~al.(2024)Ahn, Cheng, Daneshmand, and Sra]{ahn2024transformers}
Kwangjun Ahn, Xiang Cheng, Hadi Daneshmand, and Suvrit Sra.
\newblock Transformers learn to implement preconditioned gradient descent for in-context learning.
\newblock \emph{Advances in Neural Information Processing Systems}, 36, 2024.

\bibitem[Mahankali et~al.(2023)Mahankali, Hashimoto, and Ma]{mahankali2023one}
Arvind Mahankali, Tatsunori~B Hashimoto, and Tengyu Ma.
\newblock One step of gradient descent is provably the optimal in-context learner with one layer of linear self-attention.
\newblock \emph{arXiv preprint arXiv:2307.03576}, 2023.

\bibitem[Xie et~al.(2021)Xie, Raghunathan, Liang, and Ma]{xie2021explanation}
Sang~Michael Xie, Aditi Raghunathan, Percy Liang, and Tengyu Ma.
\newblock An explanation of in-context learning as implicit bayesian inference.
\newblock \emph{arXiv preprint arXiv:2111.02080}, 2021.

\bibitem[Jeon et~al.(2024)Jeon, Lee, Lei, and Van~Roy]{jeon2024information}
Hong~Jun Jeon, Jason~D Lee, Qi~Lei, and Benjamin Van~Roy.
\newblock An information-theoretic analysis of in-context learning.
\newblock \emph{arXiv preprint arXiv:2401.15530}, 2024.

\bibitem[Lin and Lee(2024)]{lin2024dual}
Ziqian Lin and Kangwook Lee.
\newblock Dual operating modes of in-context learning.
\newblock \emph{arXiv preprint arXiv:2402.18819}, 2024.

\bibitem[Yoo et~al.(2022)Yoo, Kim, Kim, Cho, Jo, Lee, Lee, and Kim]{yoo2022ground}
Kang~Min Yoo, Junyeob Kim, Hyuhng~Joon Kim, Hyunsoo Cho, Hwiyeol Jo, Sang-Woo Lee, Sang-goo Lee, and Taeuk Kim.
\newblock Ground-truth labels matter: A deeper look into input-label demonstrations.
\newblock \emph{arXiv preprint arXiv:2205.12685}, 2022.

\bibitem[Huang et~al.(2023)Huang, Cheng, and Liang]{huang2023context}
Yu~Huang, Yuan Cheng, and Yingbin Liang.
\newblock In-context convergence of transformers.
\newblock \emph{arXiv preprint arXiv:2310.05249}, 2023.

\bibitem[Dan et~al.(2020)Dan, Wei, and Ravikumar]{dan2020sharp}
Chen Dan, Yuting Wei, and Pradeep Ravikumar.
\newblock Sharp statistical guaratees for adversarially robust gaussian classification.
\newblock In \emph{International Conference on Machine Learning}, pages 2345--2355. PMLR, 2020.

\bibitem[Taheri et~al.(2022)Taheri, Pedarsani, and Thrampoulidis]{taheri2022asymptotic}
Hossein Taheri, Ramtin Pedarsani, and Christos Thrampoulidis.
\newblock Asymptotic behavior of adversarial training in binary linear classification.
\newblock In \emph{2022 IEEE International Symposium on Information Theory (ISIT)}, pages 127--132. IEEE, 2022.

\bibitem[Wang and Thrampoulidis(2021)]{wang2021benign}
Ke~Wang and Christos Thrampoulidis.
\newblock Benign overfitting in binary classification of gaussian mixtures.
\newblock In \emph{ICASSP 2021-2021 IEEE International Conference on Acoustics, Speech and Signal Processing (ICASSP)}, pages 4030--4034. IEEE, 2021.

\bibitem[Carlini et~al.(2024)Carlini, Jagielski, Choquette-Choo, Paleka, Pearce, Anderson, Terzis, Thomas, and Tram{\`e}r]{carlini2024poisoning}
Nicholas Carlini, Matthew Jagielski, Christopher~A Choquette-Choo, Daniel Paleka, Will Pearce, Hyrum Anderson, Andreas Terzis, Kurt Thomas, and Florian Tram{\`e}r.
\newblock Poisoning web-scale training datasets is practical.
\newblock In \emph{2024 IEEE Symposium on Security and Privacy (SP)}, pages 407--425. IEEE, 2024.

\bibitem[He et~al.(2024)He, Xu, Xing, Liu, Yamada, and Tang]{he2024data}
Pengfei He, Han Xu, Yue Xing, Hui Liu, Makoto Yamada, and Jiliang Tang.
\newblock Data poisoning for in-context learning.
\newblock \emph{arXiv preprint arXiv:2402.02160}, 2024.

\bibitem[Gao et~al.(2023)Gao, Xiong, Gao, Jia, Pan, Bi, Dai, Sun, and Wang]{gao2023retrieval}
Yunfan Gao, Yun Xiong, Xinyu Gao, Kangxiang Jia, Jinliu Pan, Yuxi Bi, Yi~Dai, Jiawei Sun, and Haofen Wang.
\newblock Retrieval-augmented generation for large language models: A survey.
\newblock \emph{arXiv preprint arXiv:2312.10997}, 2023.

\bibitem[Xi et~al.(2023)Xi, Chen, Guo, He, Ding, Hong, Zhang, Wang, Jin, Zhou, et~al.]{xi2023rise}
Zhiheng Xi, Wenxiang Chen, Xin Guo, Wei He, Yiwen Ding, Boyang Hong, Ming Zhang, Junzhe Wang, Senjie Jin, Enyu Zhou, et~al.
\newblock The rise and potential of large language model based agents: A survey.
\newblock \emph{arXiv preprint arXiv:2309.07864}, 2023.

\bibitem[Abdallah et~al.(2016)Abdallah, Maarof, and Zainal]{abdallah2016fraud}
Aisha Abdallah, Mohd~Aizaini Maarof, and Anazida Zainal.
\newblock Fraud detection system: A survey.
\newblock \emph{Journal of Network and Computer Applications}, 68:\penalty0 90--113, 2016.

\bibitem[Vaswani et~al.(2017)Vaswani, Shazeer, Parmar, Uszkoreit, Jones, Gomez, Kaiser, and Polosukhin]{vaswani2017attention}
Ashish Vaswani, Noam Shazeer, Niki Parmar, Jakob Uszkoreit, Llion Jones, Aidan~N Gomez, {\L}ukasz Kaiser, and Illia Polosukhin.
\newblock Attention is all you need.
\newblock \emph{Advances in neural information processing systems}, 30, 2017.

\bibitem[Radford et~al.(2019)Radford, Wu, Child, Luan, Amodei, Sutskever, et~al.]{radford2019language}
Alec Radford, Jeffrey Wu, Rewon Child, David Luan, Dario Amodei, Ilya Sutskever, et~al.
\newblock Language models are unsupervised multitask learners.
\newblock \emph{OpenAI blog}, 1\penalty0 (8):\penalty0 9, 2019.

\bibitem[Biderman et~al.(2023)Biderman, Schoelkopf, Anthony, Bradley, O’Brien, Hallahan, Khan, Purohit, Prashanth, Raff, et~al.]{biderman2023pythia}
Stella Biderman, Hailey Schoelkopf, Quentin~Gregory Anthony, Herbie Bradley, Kyle O’Brien, Eric Hallahan, Mohammad~Aflah Khan, Shivanshu Purohit, USVSN~Sai Prashanth, Edward Raff, et~al.
\newblock Pythia: A suite for analyzing large language models across training and scaling.
\newblock In \emph{International Conference on Machine Learning}, pages 2397--2430. PMLR, 2023.

\bibitem[Touvron et~al.(2023)Touvron, Martin, Stone, Albert, Almahairi, Babaei, Bashlykov, Batra, Bhargava, Bhosale, et~al.]{touvron2023llama}
Hugo Touvron, Louis Martin, Kevin Stone, Peter Albert, Amjad Almahairi, Yasmine Babaei, Nikolay Bashlykov, Soumya Batra, Prajjwal Bhargava, Shruti Bhosale, et~al.
\newblock Llama 2: Open foundation and fine-tuned chat models.
\newblock \emph{arXiv preprint arXiv:2307.09288}, 2023.

\bibitem[Wang et~al.(2018)Wang, Singh, Michael, Hill, Levy, and Bowman]{wang2018glue}
Alex Wang, Amanpreet Singh, Julian Michael, Felix Hill, Omer Levy, and Samuel~R Bowman.
\newblock Glue: A multi-task benchmark and analysis platform for natural language understanding.
\newblock \emph{arXiv preprint arXiv:1804.07461}, 2018.

\bibitem[Bengio et~al.(2009)Bengio, Louradour, Collobert, and Weston]{bengio2009curriculum}
Yoshua Bengio, J{\'e}r{\^o}me Louradour, Ronan Collobert, and Jason Weston.
\newblock Curriculum learning.
\newblock In \emph{Proceedings of the 26th annual international conference on machine learning}, pages 41--48, 2009.

\bibitem[Elman(1993)]{elman1993learning}
Jeffrey~L Elman.
\newblock Learning and development in neural networks: The importance of starting small.
\newblock \emph{Cognition}, 48\penalty0 (1):\penalty0 71--99, 1993.

\end{thebibliography}

\newpage

\appendix
The structure of the appendix is as follows. In Section \ref{sec:appendix:noncentral}, we provide the discussion when the decision boundary in Lemma \ref{lem:decision} is not a hyperplane. Section \ref{sec:appendix:proof} collects the proof for all lemmas and theorems in the main content. Section \ref{sec:appendix:simulation setup} describes the simulation setups, and Section \ref{sec:appendix:experiments} includes additional experiment results corresponding to Section \ref{sec:experiment} and Section \ref{sec:dependent} for both independent and dependent scenarios.

\section{Technical Discussion when $\sigma_+^2\neq\sigma_-^2$}\label{sec:appendix:noncentral}

\textbf{General scenario.}~~In Theorem \ref{thm:icl_acc} and the propositions in Section \ref{sec:demonstration}, our assumptions aim to simplify the analysis so that the decision boundary is a hyperplane. When the examples are provided such that $\sigma_{+}^2\neq\sigma_{-}^2$, the general intuition holds, but the decision boundary changes from a hyperplane into a sphere. To be specific, based on Lemma \ref{lem:decision}, we have
$$N(x)=\left(\sqrt{\frac{\sigma_-^2+\sigma_{\theta_-}^2}{\sigma_+^2+\sigma_{\theta_+}^2}}\right)^m \exp \left[-\frac{(x-\hat{\theta}_+)^T(x-\hat{\theta}_+)}{2(\sigma^2_++\sigma^2_{\theta_+})}+\frac{(x-\hat{\theta}_-)^T(x-\hat{\theta}_-)}{2(\sigma^2_-+\sigma^2_{\theta_-})}\right].$$
Then for some constant $a$,
\begin{eqnarray*}
    &&N(x)>a\\
    &\Leftrightarrow& -\frac{(x-\hat{\theta}_+)^T(x-\hat{\theta}_+)}{2(\sigma^2_++\sigma^2_{\theta_+})}+\frac{(x-\hat{\theta}_-)^T(x-\hat{\theta}_-)}{2(\sigma^2_-+\sigma^2_{\theta_-})}>\log(a)-m\log\left(\sqrt{\frac{\sigma_-^2+\sigma_{\theta_-}^2}{\sigma_+^2+\sigma_{\theta_+}^2}}\right),
\end{eqnarray*}
where the decision boundary
$$-\frac{(x-\hat{\theta}_+)^T(x-\hat{\theta}_+)}{2(\sigma^2_++\sigma^2_{\theta_+})}+\frac{(x-\hat{\theta}_-)^T(x-\hat{\theta}_-)}{2(\sigma^2_-+\sigma^2_{\theta_-})}=\log(a)-m\log\left(\sqrt{\frac{\sigma_-^2+\sigma_{\theta_-}^2}{\sigma_+^2+\sigma_{\theta_+}^2}}\right)$$
is a sphere.

In such a case, in Theorem \ref{thm:icl_acc}, for
\begin{equation*}
\begin{aligned}
    &P(correct|y=+1)=\int_{f_{ICL}(x)\ge 0}P(x|y=+1)dx,
\end{aligned}
\end{equation*}
instead of directly using $\Phi$ to represent the probability, we use the noncentral Chi-square distribution to write the probability. The following is the definition of noncentral Chi-square distribution.
\begin{definition}[Noncentral Chi-square distribution\footnote{\url{https://en.wikipedia.org/wiki/Noncentral_chi-squared_distribution}}]
    Let $(X_1,X_2,\ldots,X_k)$ be $k$ independent and normally distributed random variables with means $\mu_i$ and unit variances. Then the random variable 
    $$\sum_{i=1}^k X_i^2$$
    is distributed according to the noncentral chi-square distribution. It has two parameters: $k$ which specifies the number of degrees of freedom and $\lambda$ which is related to the mean of $X_i$s by
    $$\lambda=\sum_{i=1}^k\mu_i^2.$$
\end{definition}

\blue{In the following, we provide additional results corresponding to Section \ref{sec:demonstration}.

\noindent\textbf{Contradicting knowledge.} We take $k\rightarrow\infty$ and assume $\sigma_{e+}^2,\sigma_{e-}^2\rightarrow 0$. When $k\sigma_M^2 \gg\sigma_+^2$ and $k\sigma_M^2\gg\sigma_-^2$, $N(x)$ converges to the following:
\begin{eqnarray*}
    \lim_{k\rightarrow\infty}N(x)=\left(\frac{\sigma_-^2}{\sigma_+^2}\right)^{m/2}\exp\left[ -\frac{\|x-\theta_+^e\|^2}{2\sigma_+^2}+ \frac{\|x-\theta_-^e\|^2}{2\sigma_-^2}\right].
\end{eqnarray*}
When $k\sigma_M^2 \ll\sigma_+^2$ and $k\sigma_M^2\ll\sigma_-^2$, it becomes
\begin{eqnarray*}
    \lim_{k\rightarrow\infty}N(x)=\left(\frac{\sigma_-^2}{\sigma_+^2}\right)^{m/2}\exp\left[ -\frac{\|x-\theta_M\|^2}{4\sigma_+^2}+ \frac{\|x+\theta_M\|^2}{4\sigma_-^2}\right].
\end{eqnarray*}
From the above results, the implication of Proposition \ref{prop:contradict} still holds: when there is no enough examples, the decision mainly follows the prior knowledge. When there are enough examples, the knowledge from the examples will dominate and determine $N(x)$.

\noindent\textbf{Imbalanced examples.} We take $k\rightarrow\infty$ and assume all $\sigma^2,\sigma_+^2,\sigma_-^2$ are constants. Then when $\pi\rightarrow 1$, recall that
$\hat{y}_{ICL}=1(f_{ICL}(x)>0)$, where $f_{ICL}(x)=N(x)-z_k$. In this case, $z_k\rightarrow 0$ and $N(x)$ will be always larger than $0$, i.e., the ICL will always output $\hat{y}_{ICL}=+1$.

\noindent\textbf{Label noise.} When $\sigma_+^2\neq\sigma_-^2$, since the decision boundary is no longer a hyperplane, it is hard to exactly delineate the ICL accuracy into different regimes. Assume that $k\rightarrow\infty$, $\pi=1/2$.
If the noise level is large so that $\|\mathbb{E}\hat\theta_+-\mathbb{E}\hat\theta_-\|\rightarrow 0$, then in this case, denoting $\mathbb{E}\hat\theta_+=\theta$, we have
\begin{eqnarray*}
    \lim_{k\rightarrow\infty}N(x)=\left(\frac{\sigma_-^2}{\sigma_+^2}\right)^{m/2}\exp\left[ -\frac{\|x-\theta\|^2}{2\sigma_+^2}+ \frac{\|x-\theta\|^2}{2\sigma_-^2}\right].
\end{eqnarray*}
When $\sigma_-^2>\sigma_+^2$, if $\|x-\theta\|$ is small, then the prediction is $+1$. If $\|x-\theta\|$ is large, then the prediction is $-1$. For other cases of $p_e^+$ and $p_e^-$, whether the prediction is more likely to be $+1$ or $-1$ depends on the exact value of $x$.
}
\section{Proofs}\label{sec:appendix:proof}

\subsection{Proof of Lemma \ref{lem:posterior_param}}\label{sec:appendix:proof:posterior_param}

\begin{proof}[Proof of Lemma \ref{lem:posterior_param}]
Given the prior distribution of $\theta_-,\theta_+,\pi$ and the data $\{(x_i,y_i)\}$, the likelihood becomes
\begin{equation}
    \begin{aligned}
        &P(\theta_+,\theta_-,\pi|(x_1,y_1),...,(x_k,y_k),M)\\
        &\propto P((x_1,y_1),...,(x_k,y_k)|\theta_+,\theta_-, \pi,M)P(\theta_+,\theta_-,\pi|M)\\
        &=P(\theta_+,\theta_-,\pi)\prod_{i=1}^k P((x_i,y_i)|\theta_+,\theta_-,\pi) ~~\text{(Omit $M$ for simplicity)}\\
        &=P(\theta_+,\theta_-, \pi)\\
        &\qquad\cdot\prod_{y_i=+1}\pi\frac{1}{\left(\sqrt{2\pi\sigma_+^2}\right)^m}\exp{\left[-\frac{1}{2\sigma_+^2}(x_i-\theta_+)^\top(x_i-\theta_+)\right]}\\
        &\qquad\cdot\prod_{y_i=-1} (1-\pi)\frac{1}{\left(\sqrt{2\pi\sigma_+^2}\right)^m}\exp{\left[-\frac{1}{2\sigma_+^2}(x_i-\theta_+)^\top(x_i-\theta_+)\right]}\\
        &=[\pi^{\#(y_i=+1)}(1-\pi)^{\#(y_i=-1)}]P(\theta_+,\theta_-)\prod_{y_i=+1}\pi\frac{1}{\left(\sqrt{2\pi\sigma_+^2}\right)^m}\exp{\left[-\frac{1}{2\sigma_+^2}(x_i-\theta_+)^\top(x_i-\theta_+)\right]}\\
        &\qquad\cdot\prod_{y_i=-1}(1-\pi)\frac{1}{\left(\sqrt{2\pi\sigma_+^2}\right)^m}\exp{\left[-\frac{1}{2\sigma_+^2}(x_i-\theta_+)^\top(x_i-\theta_+)\right]}
    \end{aligned}.
\end{equation}

\textbf{Posterior of $\pi$}.
Since all parameters are independent, we can obtain the posterior distribution of $\pi$ as
$$
    P(\pi|\{(x_i,y_i)\},M)\propto P(\pi|M)\pi^{\#(y_i=+1)}(1-\pi)^{\#(y_i=-1)}\propto \pi^{\#(y_i=+1)}(1-\pi)^{\#(y_i=-1)}.
$$

Therefore, the posterior of $\pi$ is $Beta(\#(y_i=+1)+1, \#(y_i=-1)k+1)$.

\textbf{Posterior of $\theta_+,\theta_-$}. The likelihood of $\theta_+$ satisfies
$$
\begin{aligned}
    P(\theta_+|\{(x_i,y_i)\}|M)&\propto P(\theta_+|M)\prod_{y_i=+1}\pi\frac{1}{\left(\sqrt{2\pi\sigma_+^2}\right)^m}\exp{\left[-\frac{1}{2\sigma_+^2}(x_i-\theta_+)^\top(x_i-\theta_+)\right]}\\
    &\propto \exp \left[-\frac{1}{2\sigma_M^2}(\theta_+-\theta_M)^\top(\theta_+-\theta_M)-\frac{1}{2\sigma_+^2}\sum_{y_i=+1}(x_i-\theta_+)^\top(x_i-\theta_+)\right],
\end{aligned}
$$
which means that the posterior of $\theta_+$ follows a Gaussian distribution, i.e. $$\theta_+\sim N\left(\frac{\sigma_+^2\theta_M+\sigma_M^2\sum_{y_i=+1}x_i}{\sigma_+^2+\#(y_i=+1)\sigma^2_M},\frac{\sigma_+^2\sigma_M^2}{\sigma_+^2+\#(y_i=+1)\sigma^2_M}I\right)=N(\hat{\theta}_+,\sigma^2_{\theta_+}I).$$
Similarly, the posterior of $\theta_-$ follows 
$$\theta_-\sim N\left(\frac{\sigma_M^2\sum_{y_i=-1}x_i-\sigma_-^2\theta_M}{\sigma_-^2+\#(y_i=-1)\sigma_M^2},\frac{\sigma_M^2\sigma_-^2}{\sigma_-^2+\#(y_i=-1)\sigma_M^2}I\right)=N(\hat{\theta}_-,\sigma^2_{\theta_-}I).$$

\end{proof}

\subsection{Proof of Lemma \ref{lem:decision}}\label{sec:appendix:proof:decision}
\begin{proof}[Proof of Lemma \ref{lem:decision}]

Given $p_+=p_-=1$ and the posterior of $\theta_+,\theta_-,\pi$, we obtain that
\begin{eqnarray*}
    &&P(x,y=+1|\{(x_i,y_i)\},M)\\
    &=&\int_{\pi}\int_{\theta_+,\theta_-}\pi \frac{1}{\left(\sqrt{2\pi \sigma_+^2}\right)^m}\exp \left[-\frac{1}{2\sigma_+^2}(x-\theta_+)^\top(x-\theta_+)\right]\\
    &&\qquad\qquad\cdot P(\pi|\{(x_i,y_i)\},M)P(\theta_+|\{(x_i,y_i)\},M)P(\theta_-|\{(x_i,y_i)\},M) d\pi d\theta_+ d\theta_-\\
    &=&\int_{\pi}\int_{\theta_+}\pi \frac{1}{\left(\sqrt{2\pi \sigma_+^2}\right)^m}\exp \left[-\frac{1}{2\sigma_+^2}(x-\theta_+)^\top(x-\theta_+)\right]\\
    &&\qquad\qquad\cdot P(\pi|\{(x_i,y_i)\},M)P(\theta_+|\{(x_i,y_i)\},M) d\pi d\theta_+ \\
    &=&\int_{\pi}\int_{\theta_+}\pi \frac{1}{\left(\sqrt{2\pi \sigma_+^2}\right)^m\left(\sqrt{2\pi \sigma_{\theta_+}^2}\right)^m}\exp \left[-\frac{1}{2\sigma_+^2}(x-\theta_+)^\top(x-\theta_+)\right]\\
    &&\qquad\qquad \cdot P(\pi|\{(x_i,y_i)\},M)\exp\left[-\frac{1}{2{\sigma}_{\theta_+}^2}(\theta_+-\hat\theta_+)^\top(\theta_+-\hat\theta_+)\right] d\pi d\theta_+ \\
    &=&\frac{\#(y_i=+1)+1}{k+2}\int_{\theta_+} \frac{1}{\left(\sqrt{2\pi \sigma_+^2}\right)^m\left(\sqrt{2\pi \sigma_{\theta_+}^2}\right)^m}\exp \left[-\frac{1}{2\sigma_+^2}(x-\theta_+)^\top(x-\theta_+)\right]\\
    &&\qquad\qquad \cdot \exp\left[-\frac{1}{2{\sigma}_{\theta_+}^2}(\theta_+-\hat\theta_+)^\top(\theta_+-\hat\theta_+)\right] d\theta_+ \\
    &=&\frac{\#(y_i=+1)+1}{k+2}\frac{1}{\left(\sqrt{2\pi (\sigma_+^2+\sigma^2_{\theta_+})}\right)^m}\exp \left[-\frac{1}{2(\sigma_+^2+\sigma^2_{\theta_+})}(x-\hat{\theta}_+)^\top(x-\hat{\theta}_+)\right].    
\end{eqnarray*}

Similarly, we have $$P(x,y=-1|\{(x_i,y_i)\},M)=\frac{\#(y_i=-1)+1}{k+2}\frac{1}{\left(\sqrt{2\pi (\sigma_-^2+\sigma^2_{\theta_-})}\right)^m}\exp \left[-\frac{(x-\hat{\theta}_-)^\top(x-\hat{\theta}_-)}{2(\sigma_-^2+\sigma^2_{\theta_-})}\right].$$

Then we can obtain the predicted probability and decision boundary.
\begin{eqnarray*}
    P(y=+1|x,\{(x_i,y_i)\},M)&=&\frac{P(x,y=+1|\{(x_i,y_i)\},M)}{P(x,y+=1|\{(x_i,y_i)\},M)+P(x,y=-1|\{(x_i,y_i)\},M)}\\
    &=&\frac{\frac{\#(y_i=+1)+1}{k+2} N(x)}{\frac{\#(y_i=+1)+1}{k+2} N(x)+\frac{\#(y_i=-1)+1}{k+2}},\\
    P(y=-1|x,\{(x_i,y_i)\},M)&=&\frac{P(x,y=-1|\{(x_i,y_i)\},M)}{P(x,y+=1|\{(x_i,y_i)\},M)+P(x,y=-1|\{(x_i,y_i)\},M)}\\
    &=&\frac{\frac{\#(y_i=-1)+1}{k+2}}{\frac{\#(y_i=+1)+1}{k+2} N(x)+\frac{\#(y_i=-1)+1}{k+2}},
\end{eqnarray*}
where $$N(x)=\left(\sqrt{\frac{\sigma_-^2+\sigma_{\theta_-}^2}{\sigma_+^2+\sigma_{\theta_+}^2}}\right)^m \exp \left[-\frac{(x-\hat{\theta}_+)^\top(x-\hat{\theta}_+)}{2(\sigma^2_++\sigma^2_{\theta_+})}+\frac{(x-\hat{\theta}_-)^\top(x-\hat{\theta}_-)}{2(\sigma^2_-+\sigma^2_{\theta_-})}\right].$$

Finally, the decision boundary is $\hat{y}_{ICL}=1(f_{ICL}(x)>0)$, where $f_{ICL}(x)=N(x)-\frac{\#(y_i=-1)+1}{\#(y_i=+1)+1}$.
\end{proof}

\subsection{Proof of Theorem \ref{thm:icl_acc}}\label{sec:appendix:proof:icl_acc}


\begin{proof}[Proof of Theorem \ref{thm:icl_acc}]

{To prove Theorem \ref{thm:icl_acc}, a first note is that although we assume the LLM uses an idea of Bayesian inference to determine the posterior distribution of $\theta_+$, $\theta_-$ and $\pi$, when quantifying the ICL prediction accuracy at the inference stage, we assume a given $(\pi,\theta^e+,\theta^e_-,p_+,p_-)$. }

We first derive the marginal distribution for test input $x$.

    For any given $\theta_+$ and $\theta_-$, following the data generation assumption in Section \ref{sec:assumption}, we have
\begin{eqnarray*}
    P(x|y=+1)&=&\frac{1}{(\sqrt 2\pi \sigma^2)^m}\exp \left[-\frac{(x-\theta_+)^\top(x-\theta_+)}{2\sigma^2}\right],\\
    P(x|y=-1)&=&\frac{1}{(\sqrt {2\pi \sigma^2})^m}\exp \left[-\frac{(x-\theta_-)^\top(x-\theta_-)}{2\sigma^2}\right].
\end{eqnarray*}
As a result, when integrating over all possible $\theta_+$ and $\theta_-$, it becomes
\begin{eqnarray*}
    P(x|y=+1)&=&\int_{\theta_+,\theta_-}P(x|y=+1,\theta_+,\theta_-)P(\theta_+)P(\theta_-)d\theta_+d\theta_-\\
    &=&\int_{\theta_+,\theta_-}\frac{1}{(\sqrt{2\pi \sigma^2})^m}\exp \left[-\frac{1}{2\sigma^2}(x-\theta_+)^\top(x-\theta_+)\right]\\
    &&\qquad\qquad\cdot \frac{1}{\left(\sqrt{2\pi \sigma^2_{e+}}\right)^m}\exp \left[-\frac{1}{2\sigma^2_{e+}}(\theta_+-\theta_+^e)^\top(\theta_+-\theta_+^e)\right]\\
    &&\qquad\qquad\cdot\frac{1}{\left(\sqrt {2\pi \sigma^2_{e-}}\right)^m}\exp \left[-\frac{1}{2\sigma^2_{e-}}(\theta_--\theta_-^e)^\top(\theta_--\theta_-^e)\right] d\theta_+d\theta_-\\
&=&\int_{\theta_+}\frac{1}{(\sqrt{2\pi \sigma^2})^m}\exp \left[-\frac{1}{2\sigma^2}(x-\theta_+)^\top(x-\theta_+)\right]\\
    &&\qquad\qquad\cdot \frac{1}{\left(\sqrt{2\pi \sigma^2_{e+}}\right)^m}\exp \left[-\frac{1}{2\sigma^2_{e+}}(\theta_+-\theta_+^e)^\top(\theta_+-\theta_+^e)\right]d\theta_+\\
    &=&\frac{1}{\left(\sqrt{2\pi(\sigma^2+\sigma^2_{e+})}\right)^m}\exp \left[-\frac{1}{2(\sigma^2+\sigma^2_{e+})}(x-\theta_+^e)^\top(x-\theta_+^e)\right].
\end{eqnarray*}
Similarly,
\begin{eqnarray*}
    P(x|y=-1)&=&
    \frac{1}{\left(\sqrt{2\pi(\sigma^2+\sigma^2_{e-})}\right)^m}\exp \left[-\frac{1}{2(\sigma^2+\sigma^2_{e-})}(x-\theta_-^e)^\top(x-\theta_-^e)\right].
\end{eqnarray*}
Then, we compute the probability of correct prediction for each class respectively.
\begin{equation*}
\begin{aligned}
    &P(correct|y=+1, \{(x_i,y_i)\},\Theta,M)=\int_{f_{ICL}(x)\ge 0}P(x|y=+1)dx.
\end{aligned}
\end{equation*}
Recall that $f_{ICL}(x)=N(x)-\frac{\#(y_i=-1)+1}{\#(y_i=+1)+1}$, and
$$N(x)=\left(\sqrt{\frac{\sigma_-^2+\sigma_{\theta_-}^2}{\sigma_+^2+\sigma_{\theta_+}^2}}\right)^m \exp \left[-\frac{(x-\hat{\theta}_+)^\top(x-\hat{\theta}_+)}{2(\sigma^2_++\sigma^2_{\theta_+})}+\frac{(x-\hat{\theta}_-)^\top(x-\hat{\theta}_-)}{2(\sigma^2_-+\sigma^2_{\theta_-})}\right].$$
A difficulty of solving $f_{ICL}(x)>0$ is that, when $\sigma_-^2+\sigma_{\theta_-}^2\neq \sigma_-^2+\sigma_{\theta_-}^2$, the region of $\log(N(x))=\log[(\#(y_i=-1)+1)/\#(y_i=+1)+1]$ is not a hyperplane, and is not traceable. To overcome this difficulty, we provide an upper bound of a lower bound of $\log(N(x))$, both of which are hyperplane and can together bound $\int_{f_{ICL}(x)\ge 0}P(x|y=+1)dx$ with probability tending to 1.

To bound $N(x)$, since $\sigma_+^2=\sigma_-^2=\sigma^2$, we have the following decomposition
\begin{eqnarray*}
    \log(N(x))&=& \frac{m}{2}\log\left(\frac{\sigma_-^2+\sigma_{\theta_-}^2}{\sigma_+^2+\sigma_{\theta_+}^2}\right)-\frac{(x-\hat{\theta}_+)^\top(x-\hat{\theta}_+)}{2(\sigma^2_++\sigma^2_{\theta_+})}+\frac{(x-\hat{\theta}_-)^\top(x-\hat{\theta}_-)}{2(\sigma^2_-+\sigma^2_{\theta_-})}\\
    &=&\underbrace{\frac{m}{2}\log\left(\frac{\sigma_-^2+\sigma_{\theta_-}^2}{\sigma_+^2+\sigma_{\theta_+}^2}\right)-\frac{(x-\hat{\theta}_+)^\top(x-\hat{\theta}_+)}{2\sigma^2_+}+\frac{(x-\hat{\theta}_-)^\top(x-\hat{\theta}_-)}{2\sigma^2_-}}_{\triangleq \log(\tilde{N}(x))}\\
    &&+\underbrace{\frac{(x-\hat{\theta}_+)^\top(x-\hat{\theta}_+)}{2\sigma^2_+}-\frac{(x-\hat{\theta}_+)^\top(x-\hat{\theta}_+)}{2(\sigma^2_++\sigma^2_{\theta_+})}-\frac{(x-\hat{\theta}_-)^\top(x-\hat{\theta}_-)}{2\sigma^2_-}+\frac{(x-\hat{\theta}_-)^\top(x-\hat{\theta}_-)}{2(\sigma^2_-+\sigma^2_{\theta_-})}}_{\triangleq R(x)}.
\end{eqnarray*}
Recall that $$z_k= \log\left(\frac{\#(y_i=-1)+1}{\#(y_i=+1)+1}\right),$$
then we have
\begin{eqnarray*}
    &&P(correct|y=+1, \{(x_i,y_i)\},\Theta,M)\\
    &=&\int_{\log(N(x))-z_k \ge 0}P(x|y=+1)dx\\
    &=&\underbrace{\int_{\log(\tilde N(x))-z_k \ge 0}P(x|y=+1)dx}_{\triangleq A_1}\\
    &&-\underbrace{\int_{\log( N(x))-z_k < 0, \log(\tilde N(x))-z_k \ge 0}P(x|y=+1)dx}_{\triangleq A_2}
    +\underbrace{\int_{\log( N(x))-z_k \ge 0, \log(\tilde N(x))-z_k < 0}P(x|y=+1)dx}_{\triangleq A_3},
\end{eqnarray*}
where
\begin{eqnarray*}
    |A_2|&=&\left|\int_{\log( N(x))-z_k < 0, \log(\tilde N(x))-z_k \ge 0}P(x|y=+1)dx\right|\\
    &\leq& \int_{\log( \tilde N(x))-|R(x)|-z_k < 0, \log(\tilde N(x))-z_k \ge 0}P(x|y=+1)dx\\
    &\leq& \int_{\log( \tilde N(x))-c\log k/k-z_k < 0, \log(\tilde N(x))-z_k \ge 0}P(x|y=+1)dx + P\left(|R(x)|>\frac{c\log k}{k}\bigg|\Theta\right),
\end{eqnarray*}
and
\begin{eqnarray*}
    |A_3|&=&\left|\int_{\log( N(x))-z_k \ge 0, \log(\tilde N(x))-z_k < 0}P(x|y=+1)dx\right|\\
    &\leq&\int_{\log( \tilde N(x))+c\log k/k-z_k \ge 0, \log(\tilde N(x))-z_k < 0}P(x|y=+1)dx + P\left(|R(x)|>\frac{c\log k}{k}\bigg|\Theta\right).
\end{eqnarray*}
In the folllowing, we calculate $A_1$ and provide the detailed bounds for $A_2$ and $A_3$.

\textbf{For $A_1$,} based on the definition of $m_k$, we have
\begin{eqnarray*}
    \{x: \;\log(\tilde N(x))-z_k\geq 0\}&=&\left\{x:(\hat{\theta}_+-\hat{\theta}_-)^\top x-\frac{\hat{\theta}_+^\top\hat{\theta}_+-\hat{\theta}_-^\top\hat{\theta}_-}{2}\ge \sigma^2\left(z_k-\frac{m}{2}\log\left(\frac{\sigma_-^2+\sigma_{\theta_-}^2}{\sigma_+^2+\sigma_{\theta_+}^2}\right)\right)\right\}\\
    &=&\left\{x:(\hat{\theta}_+-\hat{\theta}_-)^\top x-\frac{\hat{\theta}_+^\top\hat{\theta}_+-\hat{\theta}_-^\top\hat{\theta}_-}{2}\ge m_k\right\}.
\end{eqnarray*}
Let $z=(\hat{\theta}_--\hat{\theta}_+)^\top x-\frac{\hat{\theta}_+^\top\hat{\theta}_+-\hat{\theta}_-^\top\hat{\theta}_-}{2}$, then we have 
\begin{eqnarray*}
z|y=+1,\{(x_i,y_i)\},M&\sim&N\left((\hat{\theta}_+-\hat{\theta}_-)^\top\theta_+^e-\frac{\hat{\theta}_+^\top\hat{\theta}_+-\hat{\theta}_-^\top\hat{\theta}_-}{2}, \|\hat{\theta}_--\hat{\theta}_+\|^2_2(\sigma^2+\sigma^2_{e+})\right),
\end{eqnarray*}
which is still a Gaussian distribution. Therefore, we have
\begin{eqnarray*}
    \int_{z\ge m_k}P(z|y=+1)dz
    &=&\left(1-\Phi\left(\frac{m_k-(\hat{\theta}_++\hat{\theta}_-)^\top\theta_+^e-\frac{\hat{\theta}_+^\top\hat{\theta}_+-\hat{\theta}_-^\top\hat{\theta}_-}{2}}{\|\hat{\theta}_--\hat{\theta}_+\|_2\sqrt{(\sigma^2+\sigma^2_{e+})}}\right)\right)\\
    &=&\left(1-\Phi\left(  \frac{m_k}{\sqrt{(\sigma^2+\sigma^2_{e+})}\|\hat{\theta}_--\hat{\theta}_+\|}-
    \frac{\theta_+^e-(\hat\theta_++\hat\theta_-)/2}{\sqrt{(\sigma^2+\sigma^2_{e+})}}\frac{\hat{\theta}_--\hat{\theta}_+}{\|\hat{\theta}_--\hat{\theta}_+\|}
    \right)\right),
    \end{eqnarray*}
    and
    \begin{eqnarray*}
    \int_{z< m_k}P(z|y=-1)dz
    &=&
    \Phi\left(\frac{m_k-(\hat{\theta}_++\hat{\theta}_-)^\top\theta_-^e-\frac{\hat{\theta}_+^\top\hat{\theta}_+-\hat{\theta}_-^\top\hat{\theta}_-}{2}}{\|\hat{\theta}_--\hat{\theta}_+\|_2\sqrt{(\sigma^2+\sigma^2_{e-})}}\right)\\
    &=&
    \Phi\left(  \frac{m_k}{\sqrt{(\sigma^2+\sigma^2_{e-})}\|\hat{\theta}_--\hat{\theta}_+\|}-
    \frac{\theta_-^e-(\hat\theta_++\hat\theta_-)/2}{\sqrt{(\sigma^2+\sigma^2_{e-})}}\frac{\hat{\theta}_--\hat{\theta}_+}{\|\hat{\theta}_--\hat{\theta}_+\|}
    \right).
\end{eqnarray*}

\textbf{In terms of $A_2$ and $A_3$,} we need to figure out $P(|R(x)|>c\log k/k)$. Denote $\tilde\theta_+$ and $\tilde\theta_-$ as the center of $x_i$s for the positive and negative classes respectively (with possible flipped examples). For a fixed $x$, we have
\begin{eqnarray*}
    \|x-\hat\theta_+\|^2\left(\frac{1}{2\sigma_+^2}-\frac{1}{2(\sigma_+^2+\sigma_{\theta_+}^2)}\right)
    &\leq& 2\left(\|x-\tilde\theta_+\|^2+\|\tilde\theta_+-\hat\theta_+\|^2\right)\frac{\sigma_{\theta_+}^2}{2\sigma_+^4}\\
    &=&2\left(\|x-\tilde\theta_+\|^2+\|\tilde\theta_+-\hat\theta_+\|^2\right)\frac{\sigma_M^2}{\sigma_+^2}\frac{1}{\sigma_+^2+\#(y_i=+1)\sigma_M^2},
\end{eqnarray*}
thus for any $t>0$, given the parameters $\tilde\theta_+$, $\theta^e_+$, $\pi$, and the examples $\{(x_i,y_i)\}$,
\begin{eqnarray*}
    &&P\left( \|x-\hat\theta_+\|^2\left(\frac{1}{2\sigma_+^2}-\frac{1}{2(\sigma_+^2+\sigma_{\theta_+}^2)}\right)>t\bigg|\Theta, \{(x_i,y_i)\} \right)\\
    &=&P\left( \|x-\hat\theta_+\|^2\left(\frac{1}{2\sigma_+^2}-\frac{1}{2(\sigma_+^2+\sigma_{\theta_+}^2)}\right)>t\bigg|\Theta, \hat\theta_+, \#(y_i=+1) \right)\\
    &\leq& P\left(\left(\|x-\tilde\theta_+\|^2+\|\tilde\theta_+-\hat\theta_+\|^2\right)\frac{\sigma_M^2}{\sigma_+^2}\frac{1}{\sigma_+^2+\#(y_i=+1)\sigma_M^2}>t\bigg|\Theta,  \hat\theta_+,\#(y_i=+1)\right)\\
    &\leq&P\left(\|x-\tilde\theta_+\|^2\frac{\sigma_M^2}{\sigma_+^2}\frac{1}{\sigma_+^2+\#(y_i=+1)\sigma_M^2}>\frac{t}{2}\bigg|\Theta,  \hat\theta_+,\#(y_i=+1)\right)\\
    &&+1\left\{\|\tilde\theta_+-\hat\theta_+\|^2\frac{\sigma_M^2}{\sigma_+^2}\frac{1}{\sigma_+^2+\#(y_i=+1)\sigma_M^2}>\frac{t}{2}\right\}.
\end{eqnarray*}
Approximating the probability of a Bernoulli distribution using the standard Gaussian distribution, we have for any $t>0$,
\begin{eqnarray}\label{eqn:k_11}
    P\left( \#(y_i=+1)<t \right)=\Phi\left( \frac{t-\pi k}{\sqrt{k\pi(1-\pi)}} \right)+O\left(\frac{1}{k^{3/2}}\right),
\end{eqnarray}
which means that
\begin{eqnarray}\label{eqn:k_12}
    P\left(\#(y_i=+1)<\pi k -c_k\sqrt{k}\log k\right)&=&\Phi\left( -\frac{c_k\log k}{\sqrt{\pi(1-\pi)}} \right)+O\left(\frac{1}{k^{3/2}}\right).
\end{eqnarray}
Therefore, with probability greater than $1-\Phi\left( -c_k(\log k)/\sqrt{\pi(1-\pi)} \right)+O\left(\frac{1}{k^{3/2}}\right)$ over the randomness of $\#(y_i=+1)$, we have
\begin{eqnarray*}
    &&P\left(\|x-\tilde\theta_+\|^2\frac{\sigma_M^2}{\sigma_+^2}\frac{1}{\sigma_+^2+\#(y_i=+1)\sigma_M^2}>\frac{t}{2}\bigg|\Theta,  \hat\theta_+,\#(y_i=+1)\right)\\
    &=&P\left( \#(y_i=+1)<\frac{1}{\sigma_M^2}\left[ \frac{2\|x-\tilde\theta_+\|^2}{t}\frac{\sigma_M^2}{\sigma_+^2}-\sigma_+^2\right] \bigg|\Theta,  \hat\theta_+,\#(y_i=+1)\right)\\
    &\leq&P\left( \pi k -c_k\sqrt{k}\log k <\frac{1}{\sigma_M^2}\left[ \frac{2\|x-\tilde\theta_+\|^2}{t}\frac{\sigma_M^2}{\sigma_+^2}-\sigma_+^2\right] \bigg|\Theta,  \hat\theta_+,\#(y_i=+1)\right)\\
    &\leq&P\left( \pi k -c_k\sqrt{k}\log k <\frac{1}{\sigma_M^2}\left[ \frac{4\|x-\theta^e_+\|^2+4\|\tilde\theta_+-\theta^e_+\|^2}{t}\frac{\sigma_M^2}{\sigma_+^2}-\sigma_+^2\right] \bigg|\Theta,  \hat\theta_+,\#(y_i=+1)\right).
\end{eqnarray*}

$$\theta_+\sim N\left(\frac{\sigma_+^2\theta_M+\sigma_M^2\sum_{y_i=+1}x_i}{\sigma_+^2+\#(y_i=+1)\sigma^2_M},\frac{\sigma_+^2\sigma_M^2}{\sigma_+^2+\#(y_i=+1)\sigma^2_M}I\right)\triangleq N(\hat{\theta}_+,\sigma^2_{\theta_+}I),$$

In terms of $1\left\{\|\tilde\theta_+-\hat\theta_+\|^2\frac{\sigma_M^2}{\sigma_+^2}\frac{1}{\sigma_+^2+\#(y_i=+1)\sigma_M^2}>\frac{t}{2}\right\}$, given $\tilde\theta_+,\theta^e_+,\{(x_i,y_i)\}$, we have for some constants $c_x$ and $c_x'$ so that when $t\geq 0$ and $t\ll \#(y_i=+1)$,
\begin{eqnarray*}
    && P\left( \|(\hat\theta_+-\tilde\theta_+)\|>t \;\bigg| \#(y_i=+1) \right)\\
    &=&P\left( \frac{ \sigma_M^2 }{ \sigma_+^2+\#(y_i=+1)\sigma_M^2 }\left\|
     \sum_{y_i=+1} x_i-\mathbb{E}x_i
    \right\|>t\;\bigg| \#(y_i=+1) \right)\\
    &=&P\left( \left\|
     \sum_{y_i=+1} x_i-\mathbb{E}x_i
    \right\|>t\frac{ \sigma_+^2+\#(y_i=+1)\sigma_M^2 }{ \sigma_M^2 }\;\bigg| \#(y_i=+1) \right)\\
    &\leq& \exp\left( -\frac{  c_xt^2}{\#(y_i=+1)\mathbb{E}_{y_i=+1}\|x_i-\mathbb{E}x_i\|^2} \left( \frac{ \sigma_+^2+\#(y_i=+1)\sigma_M^2 }{ \sigma_M^2 } \right)^2 \right)\\
    &\leq&  \exp\left( - {c_x't^2 }(\pi k -c_k\sqrt{k}\log k)  \right) + 1\left\{\#(y_i=+1)<\pi k -c_k\sqrt{k}\log k\right\}.
\end{eqnarray*}

Therefore, aggregating all the above into $R(x)$, we have
\begin{eqnarray*}
    &&P(|R(x)|>t|\Theta,\{(x_i,y_i)\})\\
    &=&P\left(\left|\frac{(x-\hat{\theta}_+)^\top(x-\hat{\theta}_+)}{2\sigma^2_+}-\frac{(x-\hat{\theta}_+)^\top(x-\hat{\theta}_+)}{2(\sigma^2_++\sigma^2_{\theta_+})}-\frac{(x-\hat{\theta}_-)^\top(x-\hat{\theta}_-)}{2\sigma^2_-}+\frac{(x-\hat{\theta}_-)^\top(x-\hat{\theta}_-)}{2(\sigma^2_-+\sigma^2_{\theta_-})}\right|>t\bigg|\Theta,\{(x_i,y_i)\}\right)\\
    &\leq&P\left(\left|\frac{(x-\hat{\theta}_+)^\top(x-\hat{\theta}_+)}{2\sigma^2_+}-\frac{(x-\hat{\theta}_+)^\top(x-\hat{\theta}_+)}{2(\sigma^2_++\sigma^2_{\theta_+})}\right|>\frac{t}{2}\bigg|\Theta,\{(x_i,y_i)\}\right)\\
    &&+P\left(\left|\frac{(x-\hat{\theta}_-)^\top(x-\hat{\theta}_-)}{2\sigma^2_-}-\frac{(x-\hat{\theta}_-)^\top(x-\hat{\theta}_-)}{2(\sigma^2_-+\sigma^2_{\theta_-})}\right|>\frac{t}{2}\bigg|\Theta,\{(x_i,y_i)\}\right),
\end{eqnarray*}
where
\begin{eqnarray}\label{eqn:k_21}
   &&P\left(\left|\frac{(x-\hat{\theta}_+)^\top(x-\hat{\theta}_+)}{2\sigma^2_+}-\frac{(x-\hat{\theta}_+)^\top(x-\hat{\theta}_+)}{2(\sigma^2_++\sigma^2_{\theta_+})}\right|>\frac{t}{2}\bigg|\Theta,\{(x_i,y_i)\}\right)\\
   &\leq& P\left( \pi k -c_k\sqrt{k}\log k <\frac{1}{\sigma_M^2}\left[ \frac{4\|x-\theta^e_+\|^2+4\|\tilde\theta_+-\theta^e_+\|^2}{t}\frac{\sigma_M^2}{\sigma_+^2}-\sigma_+^2\right] \bigg|\Theta,  \hat\theta_+,\#(y_i=+1)\right)\nonumber\\
   &&+1\left\{\|\tilde\theta_+-\hat\theta_+\|^2\frac{\sigma_M^2}{\sigma_+^2}\frac{1}{\sigma_+^2+\#(y_i=+1)\sigma_M^2}>\frac{t}{2}\right\}
   +1\left\{\#(y_i=+1)<\pi k -c_k\sqrt{k}\log k\right\}\nonumber\\
   &=&O\left( P(\mathcal{X}^2_m >t\pi k(\pi k-c_k\sqrt{k}\log k) ) \right)+1\left\{\|\tilde\theta_+-\hat\theta_+\|^2\frac{\sigma_M^2}{\sigma_+^2}\frac{1}{\sigma_+^2+\#(y_i=+1)\sigma_M^2}>\frac{t}{2}\right\}\nonumber\\
   &&+1\left\{\#(y_i=+1)<\pi k -c_k\sqrt{k}\log k\right\}.\nonumber
\end{eqnarray}
And similarly,
\begin{eqnarray}\label{eqn:k_22}
    &&P\left(\left|\frac{(x-\hat{\theta}_-)^\top(x-\hat{\theta}_-)}{2\sigma^2_-}-\frac{(x-\hat{\theta}_-)^\top(x-\hat{\theta}_-)}{2(\sigma^2_-+\sigma^2_{\theta_-})}\right|>\frac{t}{2}\bigg|\Theta,\{(x_i,y_i)\}\right)\\
    &=&O\left( P(\mathcal{X}^2_m >t(1-\pi) k((1-\pi) k-c_k\sqrt{k}\log k) ) \right)+1\left\{\|\tilde\theta_--\hat\theta_-\|^2\frac{\sigma_M^2}{\sigma_-^2}\frac{1}{\sigma_-^2+\#(y_i=+1)\sigma_M^2}>\frac{t}{2}\right\}\nonumber\\
    &&
   +1\left\{\#(y_i=+1)<\pi k -c_k\sqrt{k}\log k\right\}.\nonumber
\end{eqnarray}
Finally, taking $t=c\log k/k$, we have
\begin{eqnarray*}
    &&|A_2|+|A_3|\\
    &\leq& \int_{\log( \tilde N(x))-\log k/k-z_k < 0, \log(\tilde N(x))-z_k \ge 0}P(x|y=+1)dx\\
    &&+\int_{\log( \tilde N(x))+\log k/k-z_k \ge 0, \log(\tilde N(x))-z_k < 0}P(x|y=+1)dx \\
    &&+ 2P\left(|R(x)|>t\big|\Theta,\{(x_i,y_i)\}\right)\\
    &\leq&\Phi\left(\frac{c\log k/k +m_k-(\hat{\theta}_++\hat{\theta}_-)^\top\theta_+^e-\frac{\hat{\theta}_+^\top\hat{\theta}_+-\hat{\theta}_-^\top\hat{\theta}_-}{2}}{\|\hat{\theta}_--\hat{\theta}_+\|_2\sqrt{(\sigma^2+\sigma^2_{e+})}}\right)\\
    &&-\Phi\left(\frac{m_k-c\log k/k-(\hat{\theta}_++\hat{\theta}_-)^\top\theta_+^e-\frac{\hat{\theta}_+^\top\hat{\theta}_+-\hat{\theta}_-^\top\hat{\theta}_-}{2}}{\|\hat{\theta}_--\hat{\theta}_+\|_2\sqrt{(\sigma^2+\sigma^2_{e+})}}\right)\\
    &&+O\left( P\left(\mathcal{X}^2_m >\pi k\frac{c\log k}{k}(\pi k-c_k\sqrt{k}\log k) \right) \right)+O\left( P\left(\mathcal{X}^2_m >(1-\pi) k\frac{c\log k}{k}((1-\pi) k-c_k\sqrt{k}\log k) \right) \right)\\
    &&+1\left\{\|\tilde\theta_+-\hat\theta_+\|^2\frac{\sigma_M^2}{\sigma_+^2}\frac{1}{\sigma_+^2+\#(y_i=+1)\sigma_M^2}>\frac{t}{2}\right\}\\
    &&
   +1\left\{\|\tilde\theta_--\hat\theta_-\|^2\frac{\sigma_M^2}{\sigma_+^2}\frac{1}{\sigma_-^2+\#(y_i=+1)\sigma_M^2}>\frac{t}{2}\right\}+2*1\left\{\#(y_i=+1)<\pi k -c_k\sqrt{k}\log k\right\}\\
   &=&O\left(\frac{c\log k}{k}\right)\\
   &&+O\left( P\left(\mathcal{X}^2_m >\pi k\frac{c\log k}{k}(\pi k-c_k\sqrt{k}\log k) \right) \right)+O\left( P\left(\mathcal{X}^2_m >(1-\pi) k\frac{c\log k}{k}((1-\pi) k-c_k\sqrt{k}\log k) \right) \right)\\
   &&+1\left\{\|\tilde\theta_+-\hat\theta_+\|^2\frac{\sigma_M^2}{\sigma_+^2}\frac{1}{\sigma_+^2+\#(y_i=+1)\sigma_M^2}>\frac{c\log k}{2k}\right\}\\
   &&
   +1\left\{\|\tilde\theta_--\hat\theta_-\|^2\frac{\sigma_M^2}{\sigma_+^2}\frac{1}{\sigma_-^2+\#(y_i=+1)\sigma_M^2}>\frac{c\log k}{2k}\right\}+2*1\left\{\#(y_i=+1)<\pi k -c_k\sqrt{k}\log k\right\}.
\end{eqnarray*}
\end{proof}

\subsection{Proof of Proposition \ref{prop:asymp}}\label{sec:appendix:proof:bahadur}


To simplify the notation, we use a generic notation $c$ to represent the constant used in concentration inequalities and etc.

\begin{proof}[Proof of Proposition \ref{prop:asymp}]
    
To show Proposition \ref{prop:asymp}, the key is to quantify the difference between $P(correct|y=+1, \{(x_i,y_i)\},M)$ and $P_+^*$, where
\begin{eqnarray*}
    &&P(correct|y=+1,\Theta, \{(x_i,y_i)\},M)-P_+^*\\
    &=&-\phi\left(  \frac{\sigma^2\log ((1-\pi )/\pi )}{\sqrt{(\sigma^2+\sigma^2_{e+})}\|\tilde\theta_--\tilde\theta_+\|}-
  \frac{(\theta_+^e-(\tilde\theta_++\tilde\theta_-)/2)^\top(\tilde\theta_+-\tilde\theta_-)}{\|\tilde\theta_+-\tilde\theta_-\|\sqrt{(\sigma^2+\sigma^2_{e+})}}  
    \right)\\
    &&\qquad\times\Bigg(  \underbrace{ \frac{m_k}{\sqrt{(\sigma^2+\sigma^2_{e+})}\|\hat{\theta}_--\hat{\theta}_+\|}
    -\frac{\sigma^2\log ((1-\pi )/\pi )}{\sqrt{(\sigma^2+\sigma^2_{e+})}\|\tilde\theta_--\tilde\theta_+\|}}_{\triangleq B_1}\\
    &&\qquad\qquad\qquad\qquad\qquad
    \underbrace{- \frac{\theta_+^e-(\hat\theta_++\hat\theta_-)/2}{\sqrt{(\sigma^2+\sigma^2_{e+})}}\frac{\hat{\theta}_--\hat{\theta}_+}{\|\hat{\theta}_--\hat{\theta}_+\|}+\frac{(\theta_+^e-(\tilde\theta_++\tilde\theta_-)/2)^\top(\tilde\theta_+-\tilde\theta_-)}{\|\tilde\theta_+-\tilde\theta_-\|\sqrt{(\sigma^2+\sigma^2_{e+})}}  }_{\triangleq - B_2 } \Bigg)\\
    &&+O\left(\phi'\left(  \frac{\sigma^2\log((1-\pi )/\pi )}{\sqrt{(\sigma^2+\sigma^2_{e+})}\|\theta_-^e-\theta_+^e\|}-
   \frac{(\theta_+^e-(\tilde\theta_++\tilde\theta_-)/2)^\top(\tilde\theta_+-\tilde\theta_-)}{\|\tilde\theta_+-\tilde\theta_-\|\sqrt{(\sigma^2+\sigma^2_{e+})}} \right)(B_1-B_2)^2
    \right)\\
    &&+O\left(\frac{c\log k}{k}\right)+O\left( P\left(\mathcal{X}^2_m >\frac{c\log k}{k}(\pi k-c_k\sqrt{k}\log k) \right) \right).
\end{eqnarray*}

There are two steps in the proof: (1) Use probability bounds to show that $B_1$ and $B_2$ are consistent and the remainder terms are negligible. (2) Figure out the asymptotic distribution ignoring the remainder terms.



To bound $B_1$, we need to bound $m_k$ and $\|\hat\theta_--\hat\theta_+\|$. Based on Lemma \ref{lem:m_k} below, there exists some constant $c>0$ such that when $t\gg 1/k$,
\begin{eqnarray}
    P\left(\left| m_k- \sigma^2\log\left(\frac{\pi }{1-\pi }\right)\right|\geq t \right)\leq \exp\left( -c t^2 k \right)+\frac{c}{k}.\label{eqn:m_k}
\end{eqnarray}

In terms of $\hat\theta_-$ and $\hat\theta_+$ in $B_1$, following Lemma \ref{lem:hat_theta}, we have 
\begin{eqnarray*}
     P\left(\left\| \hat\theta_+-\tilde\theta_+\right\| \geq t\right)
     \leq 2\Phi\left(ct\sqrt{k}\right)-1 + \frac{c}{k},
\end{eqnarray*}
and
\begin{eqnarray*}
    P\left(\left\| \hat\theta_--\tilde\theta_-\right\| \geq t\right)\leq 2\Phi\left(ct\sqrt{k}\right)-1 + \frac{c}{k}.
\end{eqnarray*}
Obtaining the bound for $m_k$ and $(\hat\theta_+,\hat\theta_-)$, one can further write $B_1$ as follows:
\begin{eqnarray*}
    B_1&=&\frac{m_k}{\sqrt{(\sigma^2+\sigma^2_{e+})}\|\hat{\theta}_--\hat{\theta}_+\|}
    -\frac{\sigma^2\log ((1-\pi )/\pi )}{\sqrt{(\sigma^2+\sigma^2_{e+})}\|\tilde\theta_--\tilde\theta_+\|}\\
    &=&\frac{m_k-\sigma^2\log((1-\pi )/\pi )}{ \sqrt{(\sigma^2+\sigma^2_{e+})}\|\tilde\theta_--\tilde\theta_+\| }+m_k\left(\frac{1}{\sqrt{(\sigma^2+\sigma^2_{e+})}\|\hat{\theta}_--\hat{\theta}_+\|}
    -\frac{1}{\sqrt{(\sigma^2+\sigma^2_{e+})}\|\tilde\theta_--\tilde\theta_+\|} \right)\\
    &=&\frac{m_k-\sigma^2\log((1-\pi )/\pi )}{ \sqrt{(\sigma^2+\sigma^2_{e+})}\|\tilde\theta_--\tilde\theta_+\| }+\sigma^2\log((1-\pi )/\pi )\left(\frac{1}{\sqrt{(\sigma^2+\sigma^2_{e+})}\|\hat{\theta}_--\hat{\theta}_+\|}
    -\frac{1}{\sqrt{(\sigma^2+\sigma^2_{e+})}\|\tilde\theta_--\tilde\theta_+\|} \right)\\
    &&+\underbrace{\left(m_k-\sigma^2\log((1-\pi )/\pi )\right)\left(\frac{1}{\sqrt{(\sigma^2+\sigma^2_{e+})}\|\hat{\theta}_--\hat{\theta}_+\|}
    -\frac{1}{\sqrt{(\sigma^2+\sigma^2_{e+})}\|\tilde\theta_--\tilde\theta_+\|} \right)}_{\triangleq R_1},
\end{eqnarray*}
where when $t\gg(\log k)/k$,
\begin{eqnarray*}
    P\left(|R_1|\geq t\right)\leq P(m_k>\sqrt{t})+P(\|\hat\theta_--\hat\theta_+\|>\sqrt{t})\leq \exp\left(-\frac{ct}{k}\right),
\end{eqnarray*}
which implies that $R_1$ is a remainder terms compared to the other terms in $B_1$.

After bounding the remainder term $R_1$ in $B_1$, we are now able to derive the Bahadur representation of $B_1$. For $z_k$, we have
\begin{eqnarray*}
    &&z_k-\log\left(\frac{(1-\pi )k+1}{\pi  k+1}\right)\\
    &=&\log\left( \frac{\#(y_i=-1)+1}{\#(y_i=+1)+1} \right)-\log\left(\frac{(1-\pi )k+1}{\pi  k+1}\right)\\
    &=& \frac{\#(y_i=-1)-(1-\pi )k}{\pi  k+1}\frac{\pi  k+1}{(1-\pi )k+1} + \frac{(1-\pi )k+1}{(\pi k+1)^2}(\#(y_i=+1)-\pi  k)\frac{\pi  k+1}{(1-\pi )k+1}+O\left( (\#(y_i=+1)/k-\pi)^2 \right)\\
    &=&\frac{\#(y_i=-1)-(1-\pi )k}{(1-\pi )k+1}+\frac{\#(y_i=+1)-\pi  k}{\pi  k+1}+O\left( (\#(y_i=+1)/k-\pi)^2 \right).
\end{eqnarray*}
Thus for $m_k$, we obtain that
\begin{eqnarray*}
    m_k-\sigma^2\log\left(\frac{1-\pi }{\pi }\right) = \sigma^2\left(\frac{\#(y_i=-1)-(1-\pi )k}{(1-\pi )k+1}+\frac{\#(y_i=+1)-\pi  k}{\pi  k+1}\right)+R_2,
\end{eqnarray*}
where when $t\gg (\log k)/k$, there exists some $c>0$ such that
\begin{eqnarray*}
    P(|R_2|\geq t)\leq \exp\left(-\frac{ct}{k}\right).
\end{eqnarray*}
On the other hand, for the term $\|\hat\theta_--\hat\theta_+\|$ in $B_1$, we have
\begin{eqnarray*}
    &&\|\hat\theta_--\hat\theta_+\|-\|\tilde\theta_--\tilde\theta_+\|\\
    &=& \left(\hat\theta_--\hat\theta_+-\tilde\theta_-+\tilde\theta_+\right)^{T}\frac{\tilde\theta_--\tilde\theta_+}{\|\tilde\theta_--\tilde\theta_+\|}+O(\|\hat\theta_--\hat\theta_+-\tilde\theta_-+\tilde\theta_+\|^2)\\
    &=&\left(\frac{1}{(1-\pi )k}\sum x_i 1\{y_i=-1\}- \frac{1}{\pi k}\sum x_i\{y_i=+1\}\right)^{T}\frac{\tilde\theta_--\tilde\theta_+}{\|\tilde\theta_--\tilde\theta_+\|}\\
    &&+\left(\frac{\tilde\theta_-}{(1-\pi )k}(\#(y_i=-1)-(1-\pi )k)- \frac{\tilde\theta_+}{\pi k}(\#(y_i=+1)-\pi k) \right)^{T}\frac{\tilde\theta_--\tilde\theta_+}{\|\tilde\theta_--\tilde\theta_+\|}-\|\tilde\theta_--\tilde\theta_+\|\\
    &&+O\left(\left\|\frac{1}{\pi k}\sum x_i1\{y_i=+1\}-\tilde\theta_+\right\|^2\right) + O\left( (\#(y_i=+1)-\pi k)^2 \right).
\end{eqnarray*}
As a result,
\begin{eqnarray}
    &&B_1\label{eqn:B1}\\
    &=& \frac{\sigma^2}{\sqrt{\sigma^2+\sigma_{e+}^2}\|\tilde\theta_--\tilde\theta_+\|}\left(\frac{\#(y_i=-1)-(1-\pi )k}{(1-\pi )k+1}+\frac{\#(y_i=+1)-\pi  k}{\pi  k+1}\right)\nonumber\\
    &&+\frac{\sigma^2\log((1-\pi )/\pi )}{\sqrt{\sigma^2+\sigma_{e+}^2}}\frac{1}{\|\tilde\theta_--\tilde\theta_+\|^2} \left(\frac{1}{(1-\pi )k}\sum x_i 1\{y_i=-1\}- \frac{1}{\pi k}\sum x_i\{y_i=+1\}\right)^{T}\frac{\tilde\theta_--\tilde\theta_+}{\|\tilde\theta_--\tilde\theta_+\|}\nonumber\\
    &&+\frac{\sigma^2\log((1-\pi )/\pi )}{\sqrt{\sigma^2+\sigma_{e+}^2}}\frac{1}{\|\tilde\theta_--\tilde\theta_+\|^2} \left(\frac{\tilde\theta_-}{(1-\pi )k}(\#(y_i=-1)-(1-\pi )k)- \frac{\tilde\theta_+}{\pi k}(\#(y_i=+1)-\pi k) \right)^{T}\frac{\tilde\theta_--\tilde\theta_+}{\|\tilde\theta_--\tilde\theta_+\|}\nonumber\\
    &&-\frac{\sigma^2\log((1-\pi )/\pi )}{\sqrt{\sigma^2+\sigma_{e+}^2}}\frac{1}{\|\tilde\theta_--\tilde\theta_+\|} +R_3.\nonumber
\end{eqnarray}
\begin{eqnarray*}
    P(|R_3|\geq t)\leq \exp\left(-\frac{ct}{k}\right).
\end{eqnarray*}
Similarly, for $B_2$, we have
\begin{eqnarray*}
    &&(\hat{\theta}_++\hat\theta_-)^\top\frac{\hat\theta_--\hat\theta_+}{\|\hat\theta_--\hat\theta_+\|}-(\tilde\theta_++\tilde\theta_-)^\top\frac{(\tilde\theta_--\tilde\theta_+)}{\|\tilde\theta_--\tilde\theta_+\|}\\
    &=&(\tilde\theta_++\tilde\theta_-)^\top\left( \frac{I}{\| \tilde\theta_--\tilde\theta_+\|} -\frac{(\tilde\theta_--\tilde\theta_+)(\tilde\theta_--\tilde\theta_+)^\top}{\|\tilde\theta_--\tilde\theta_+\|^3}\right)\left((\hat\theta_--\hat\theta_+)-(\tilde\theta_--\tilde\theta_+)\right)\\
    &&+\left((\hat{\theta}_++\hat\theta_-)-(\tilde\theta_++\tilde\theta_-)\right)^\top\frac{(\tilde\theta_--\tilde\theta_+)}{\|\tilde\theta_--\tilde\theta_+\|}+O\left(\left\|\frac{1}{\pi k}\sum x_i1\{y_i=+1\}-\tilde\theta_+\right\|^2\right) + O\left( (\#(y_i=+1)-\pi k)^2 \right),
\end{eqnarray*}
with
\begin{eqnarray*}
    &&(\hat\theta_--\hat\theta_+)-(\tilde\theta_--\tilde\theta_+)\\
    &=&\frac{1}{(1-\pi )k}\sum x_i 1\{y_i=-1\}- \frac{1}{\pi k}\sum x_i\{y_i=+1\}\\
    &&+\frac{\tilde\theta_-}{(1-\pi )k}(\#(y_i=-1)-(1-\pi )k)- \frac{\tilde\theta_+}{\pi k}(\#(y_i=+1)-\pi k)-(\tilde\theta_--\tilde\theta_+)\\
    &&+O\left(\left\|\frac{1}{\pi k}\sum x_i1\{y_i=+1\}-\tilde\theta_+\right\|^2\right) + O\left( (\#(y_i=+1)-\pi k)^2 \right)\\
    &\triangleq& \frac{1}{(1-\pi)k}\sum \xi^-_i - \frac{1}{\pi k}\sum \xi^+_i\\
    &&+O\left(\left\|\frac{1}{\pi k}\sum x_i1\{y_i=+1\}-\tilde\theta_+\right\|^2\right) + O\left( (\#(y_i=+1)-\pi k)^2 \right).
\end{eqnarray*}
and
\begin{eqnarray*}
    &&(\hat\theta_-+\hat\theta_+)-(\tilde\theta_-+\tilde\theta_+)\\
    &=&\frac{1}{(1-\pi )k}\sum x_i 1\{y_i=-1\}+ \frac{1}{\pi k}\sum x_i\{y_i=+1\}\\
    &&+\frac{\tilde\theta_-}{(1-\pi )k}(\#(y_i=-1)-(1-\pi )k)+ \frac{\tilde\theta_+}{\pi k}(\#(y_i=+1)-\pi k)-(\tilde\theta_-+\tilde\theta_+)\\
    &&+O\left(\left\|\frac{1}{\pi k}\sum x_i1\{y_i=+1\}-\tilde\theta_+\right\|^2\right) + O\left( (\#(y_i=+1)-\pi k)^2 \right)\\
    &\triangleq& \frac{1}{(1-\pi)k}\sum \xi^-_i + \frac{1}{\pi k}\sum \xi^+_i\\
    &&+O\left(\left\|\frac{1}{\pi k}\sum x_i1\{y_i=+1\}-\tilde\theta_+\right\|^2\right) + O\left( (\#(y_i=+1)-\pi k)^2 \right).
\end{eqnarray*}
Consequently,
\begin{eqnarray}
    B_2&=&\frac{\theta_+^e-(\hat\theta_++\hat\theta_-)/2}{\sqrt{(\sigma^2+\sigma^2_{e+})}}\frac{\hat{\theta}_--\hat{\theta}_+}{\|\hat{\theta}_--\hat{\theta}_+\|}-\frac{(\theta_+^e-(\tilde\theta_++\tilde\theta_-)/2)^\top(\tilde\theta_+-\tilde\theta_-)}{\|\tilde\theta_+-\tilde\theta_-\|\sqrt{(\sigma^2+\sigma^2_{e+})}} \nonumber\\
    &=&\frac{1}{\sqrt{\sigma^2+\sigma_{e+}^2}}\left(\theta_+^e-\frac{\tilde\theta_++\tilde\theta_-}{2}\right)^\top\left( \frac{I}{\| \tilde\theta_--\tilde\theta_+\|} -\frac{(\tilde\theta_--\tilde\theta_+)(\tilde\theta_--\tilde\theta_+)^\top}{\|\tilde\theta_--\tilde\theta_+\|^3}\right)\left(\frac{1}{(1-\pi)k}\sum \xi^-_i + \frac{1}{\pi k}\sum \xi^+_i\right)\nonumber\\
    &&-\frac{1}{2\sqrt{\sigma^2+\sigma_{e+}^2}}\frac{(\tilde\theta_--\tilde\theta_+)^{\top}}{\|\tilde\theta_--\tilde\theta_+\|}\left(\frac{1}{(1-\pi)k}\sum \xi^-_i + \frac{1}{\pi k}\sum \xi^+_i\right)\nonumber\\
    &&+O\left(\left\|\frac{1}{\pi k}\sum x_i1\{y_i=+1\}-\tilde\theta_+\right\|^2\right) + O\left( (\#(y_i=+1)-\pi k)^2 \right).\label{eqn:B2}
\end{eqnarray}

Finally, after obtaining the Bahadur representation of $B_1$ and $B_2$, one can obtain that
\begin{eqnarray*}
    &&P(correct|y=+1, \{(x_i,y_i)\},M)-P_+^*\\
    &=&-\phi\left(  \frac{\sigma^2\log ((1-\pi )/\pi )}{\sqrt{(\sigma^2+\sigma^2_{e+})}\|\tilde\theta_--\tilde\theta_+\|}-
  \frac{(\theta_+^e-(\tilde\theta_++\tilde\theta_-)/2)^\top(\tilde\theta_+-\tilde\theta_-)}{\|\tilde\theta_+-\tilde\theta_-\|\sqrt{(\sigma^2+\sigma^2_{e+})}}  
    \right)(B_1-B_2)\\
    &&+O\left(\phi'\left(  \frac{\sigma^2\log((1-\pi )/\pi )}{\sqrt{(\sigma^2+\sigma^2_{e+})}\|\theta_-^e-\theta_+^e\|}-
   \frac{(\theta_+^e-(\tilde\theta_++\tilde\theta_-)/2)^\top(\tilde\theta_+-\tilde\theta_-)}{\|\tilde\theta_+-\tilde\theta_-\|\sqrt{(\sigma^2+\sigma^2_{e+})}} \right)(B_1-B_2)^2
    \right)\\
    &&+O\left(\frac{c\log k}{k}\right)+O\left( P\left(\mathcal{X}^2_m >\frac{c\log k}{k}(\pi k-c_k\sqrt{k}\log k) \right) \right).
\end{eqnarray*}

\end{proof}

\begin{lemma}\label{lem:m_k}
    Under the conditions of Theorem \ref{thm:icl_acc}, there exists some constant $c>0$ such that when $t\gg 1/k$,
\begin{eqnarray*}
    P\left(\left| m_k- \sigma^2\log\left(\frac{\pi }{1-\pi }\right)\right|\geq t \right)\leq \exp\left( -c t^2 k \right)+\frac{c}{k}.
\end{eqnarray*}
\end{lemma}

\begin{proof}[Proof of Lemma \ref{lem:m_k}]
Using Hoeffding inequality, we know that for $t>0$,
\begin{eqnarray*}
    &&P\left(|\#(y_i=+1)-\pi k|\geq t\right)\leq 2\exp\left(-\frac{2t^2}{k}\right),\\
    &&P\left(|\#(y_i=-1)-(1-\pi ) k|\geq t\right)\leq 2\exp\left(-\frac{2t^2}{k}\right),
\end{eqnarray*}
which implies that
\begin{eqnarray}
    &&P\left( \left|z_k-\log\left(\frac{(1-\pi  )k+1}{\pi  k+1}\right)\right| \geq t\right)\nonumber\\
    &=&P\left( \left|\log\left( \frac{\#(y_i=-1)+1}{\#(y_i=+1)+1} \frac{\pi  k+1}{(1-\pi  )k+1}\right)\right| \geq t\right)\nonumber\\
    &=&P\left( \exp(-t)\leq  \frac{\#(y_i=-1)+1}{\#(y_i=+1)+1} \frac{\pi  k+1}{(1-\pi  )k+1}\leq \exp(t)\right)\nonumber\\
    &\leq&P\left(  \left| \frac{\#(y_i=-1)+1}{\#(y_i=+1)+1} \frac{\pi  k+1}{(1-\pi  )k+1}-1\right|\geq 1-\exp(-t)\right)\nonumber\\
    &\leq& P\left( \left| \frac{\#(y_i=-1)+1}{(1-\pi  )k+1}-1\right|\geq {1-\exp(-t)} \right)+P\left(  \left| \frac{\pi  k+1}{\#(y_i=+1)+1} -1\right|\geq {1-\exp(-t)}\right)\nonumber\\
    &\leq&P\left( \left|\#(y_i=-1)-(1-\pi  )k \right|\leq{1-\exp(-t)}[(1-\pi )k+1 ] \right)\nonumber\\
    &&+P\left( \left| \#(y_i=+1)-\pi  k \right|\leq {1-\exp(-t)}[\pi  k+1] \right)\nonumber\\
    &\leq& 2\exp\left(-\frac{2(1-\exp(-t))^2[(1-\pi )k+1]^2}{k}\right)+2\exp\left(-\frac{2(1-\exp(-t))^2[\pi k+1]^2}{k}\right)\nonumber\\
    &\leq&2\exp\left(-\frac{2t^2[(1-\pi )k+1]^2}{k}\right)+2\exp\left(-\frac{2t^2[\pi k+1]^2}{k}\right).\label{eqn:z_k}
\end{eqnarray}
On the other hand, in terms of $\sigma_{\theta_-}^2$ and $\sigma_{\theta_+}^2$, recall that
\begin{eqnarray*}
    \sigma_{\theta_+}^2=\frac{\sigma_+^2\sigma_M^2}{\sigma_+^2+\#(y_i=+1)\sigma_M^2},\;\sigma_{\theta_-}^2=\frac{\sigma_-^2\sigma_M^2}{\sigma_-^2+\#(y_i=-1)\sigma_M^2},
\end{eqnarray*}
thus using the probability bound for $\#(y_i=+1)$ and $\#(y_i=-1)$, we can also obtain that
\begin{eqnarray*}
    &&P\bigg( \bigg| \sigma_{\theta_+}^2-\underbrace{\frac{\sigma_+^2\sigma_M^2}{\sigma_+^2+\pi  k\sigma_M^2}}_{\triangleq (\sigma_{\theta_+}^*)^2} \bigg|\geq t \bigg)\\
    &\leq& P\left( \sigma_+^2\sigma_M^2\left(\frac{1}{\sigma_+^2+(\pi  k+\sqrt{k\log k})\sigma_M^2}\right)^2\left| \#(y_i=+1)-\pi  k \right|\geq t \right)+P\left( \left| \#(y_i=+1)-\pi  k \right|\geq \sqrt{k\log k} \right)\\
    &\leq&2\exp\left(-\frac{2t^2}{k}\frac{1}{\sigma_+^2\sigma_M^2}\left({\sigma_+^2+(\pi  k+\sqrt{k\log k})\sigma_M^2)}\right)^2\right)+\frac{2}{k},
\end{eqnarray*}
and
\begin{eqnarray*}
    P\bigg( \bigg| \sigma_{\theta_-}^2-\underbrace{\frac{\sigma_-^2\sigma_M^2}{\sigma_-^2+(1-\pi ) k\sigma_M^2}}_{\triangleq (\sigma_{\theta_-}^*)^2} \bigg|\geq t \bigg)\leq 2\exp\left(-\frac{2t^2}{k}\frac{1}{\sigma_-^2\sigma_M^2}\left({\sigma_0^2+((1-\pi ) k+\sqrt{k\log k})\sigma_M^2)}\right)^2\right)+\frac{2}{k}.
\end{eqnarray*}
Thus following similar steps as to bound $z_k$, recalling that $\sigma_+^2=\sigma_-^2=\sigma^2$, we obtain
\begin{eqnarray}
    &&P\left(\left|\log\left(\frac{\sigma^2+\sigma_{\theta_-}^2}{\sigma^2+\sigma_{\theta_+}^2}\right) -\log\left( \frac{\sigma^2+(\sigma_{\theta_-}^*)^2}{\sigma^2+(\sigma_{\theta_+}^*)^2} \right) \right|\geq t\right)\nonumber\\
    &\leq& 2\exp\left(-\frac{2t^2(\sigma^2+(\sigma_{\theta_+}^*)^2)}{k}\frac{1}{\sigma^2\sigma_M^2}\left({\sigma^2+(\pi  k+\sqrt{k\log k})\sigma_M^2)}\right)^2\right)\nonumber\\
    &&+2\exp\left(-\frac{2t^2(\sigma^2+(\sigma_{\theta_-}^*)^2)}{k}\frac{1}{\sigma^2\sigma_M^2}\left({\sigma^2+((1-\pi ) k+\sqrt{k\log k})\sigma_M^2)}\right)^2\right)+\frac{4}{k}.\label{eqn:m_k_1}
\end{eqnarray}
Besides, when $k$ is large enough, it is easy to see that for some $c>0$,
\begin{eqnarray}
    \log\left( \frac{\sigma^2+(\sigma_{\theta_-}^*)^2}{\sigma^2+(\sigma_{\theta_+}^*)^2}\right) \leq \frac{c}{k}.\label{eqn:m_k_2}
\end{eqnarray}

Aggregating (\ref{eqn:z_k}), (\ref{eqn:m_k_1}), and (\ref{eqn:m_k_2}), there exists some constant $c>0$ such that when $t\gg 1/k$,
\begin{eqnarray}
    P\left(\left| m_k- \sigma^2\log\left(\frac{\pi }{1-\pi }\right)\right|\geq t \right)\leq \exp\left( -c t^2 k \right)+\frac{c}{k}.\label{eqn:m_k}
\end{eqnarray}
\end{proof}

\begin{lemma}\label{lem:hat_theta}
    Under the conditions in Theorem \ref{thm:icl_acc},
    \begin{eqnarray*}
     P\left(\left\| \hat\theta_+-\tilde\theta_+\right\| \geq t\right)
     \leq 2\Phi\left(ct\sqrt{k}\right)-1 + \frac{c}{k},
\end{eqnarray*}
and
\begin{eqnarray*}
    P\left(\left\| \hat\theta_--\tilde\theta_-\right\| \geq t\right)\leq 2\Phi\left(ct\sqrt{k}\right)-1 + \frac{c}{k}.
\end{eqnarray*}
\end{lemma}

\begin{proof}[Proof of Lemma \ref{lem:hat_theta}]
For $\hat\theta_+$, we have
\begin{eqnarray*}
    \hat{\theta}_+=\frac{\sum_{i=1}^k x_i1\{y_i=+1\} }{\#(y_i=+1)},
\end{eqnarray*}
thus fixing $\#(y_i=+1)$ and applying Berry–Esseen inequality, we have for some constant $c>0$,
\begin{eqnarray*}
    P\left(\left\| \hat\theta_+-\tilde\theta_+\right\| \leq t\,\bigg| \,\#(y_i=+1)\right)\leq 2\Phi\left( ct\sqrt{\#(y_i=+1)} \right)-1+\frac{c}{\sqrt{\#(y_i=+1)} }.
\end{eqnarray*}
Integrating all possible $\#(y_i=+1)$, we have for some $c>0$,
\begin{eqnarray*}
     &&P\left(\left\| \hat\theta_+-\tilde\theta_+\right\| \geq t\right)\\
     &=&\sum_{\#(y_i=+1)=j}P\left(\left\| \hat\theta_+-\tilde\theta_+\right\| \leq t\,\bigg| \,\#(y_i=+1)=j\right)P(\#(y_i=+1)=j)\\
     &\leq&\sum_{|\#(y_i=+1)-\pi  k|< \sqrt{k\log k}} P\left(\left\| \hat\theta_+-\tilde\theta_+\right\| \leq t\,\bigg| \,\#(y_i=+1)=j\right)P(\#(y_i=+1)=j)\\
     &&+P(|\#(y_i=+1)-\pi  k|\geq \sqrt{k\log k})\\
     &\leq&2\Phi\left(ct\sqrt{k}\right)-1 + \frac{c}{k}.
\end{eqnarray*}
And similarly,
\begin{eqnarray*}
    P\left(\left\| \hat\theta_--\tilde\theta_-\right\| \geq t\right)\leq 2\Phi\left(ct\sqrt{k}\right)-1 + \frac{c}{k}.
\end{eqnarray*}    
\end{proof}

\subsection{Proof of Proposition \ref{prop:contradict}}\label{sec:appendix:proof:contradict}

\begin{proof}[Proof of Proposition \ref{prop:contradict}]

We first consider the case where $k\sigma_M^2\gg \sigma^2$, which can directly utilize the result from Proposition \ref{prop:asymp}. Then, we turn to the regime where $k\sigma_M^2\ll \sigma^2$.

\textbf{Regime 1: $k\sigma_M^2\gg \sigma^2$.} Recall that
\begin{eqnarray*}
    &&P_+^*=
    \left(1-\Phi\left(  \frac{\sigma^2\log ((1-\pi)/\pi)}{\sqrt{(\sigma^2+\sigma^2_{e+})}\|\tilde\theta_--\tilde\theta_+\|}-
    \frac{(\theta_+^e-(\tilde\theta_++\tilde\theta_-)/2)^\top(\tilde\theta_+-\tilde\theta_-)}{\|\tilde\theta_+-\tilde\theta_-\|\sqrt{(\sigma^2+\sigma^2_{e+})}} 
    \right)\right),
\end{eqnarray*}
and
\begin{eqnarray*}
    &&P_-^*= 
    \Phi\left(  \frac{\sigma^2\log ((1-\pi)/\pi)}{\sqrt{(\sigma^2+\sigma^2_{e-})}\|\tilde\theta_--\tilde\theta_+\|}-
  \frac{(\theta_-^e-(\tilde\theta_++\tilde\theta_-)/2)^\top(\tilde\theta_+-\tilde\theta_-)}{\|\tilde\theta_+-\tilde\theta_-\|\sqrt{(\sigma^2+\sigma^2_{e-})}} 
    \right),
\end{eqnarray*}
and based on Proposition \ref{prop:asymp}, 
$$P(correct|y=+1,\Theta, \{(x_i,y_i)\},M)\rightarrow P_+^*,$$
$$P(correct|y=-1,\Theta, \{(x_i,y_i)\},M)\rightarrow P_-^*.$$
Therefore, we directly check how contradict and matched knowledge affects $P_+^*$ and $P_-^*$ respectively.

\begin{itemize}
    \item Contradict knowledge: When $\pi=1/2$, $p_+=p_-=1$, and $\theta_-^e=\theta_M=-\theta_+^e$, we have
\begin{eqnarray*}
    P_+^*= 1-\Phi \left(\frac{-\|\theta_M\|_2}{\sqrt{\sigma^2+\sigma_{e+}^2}}\right).
\end{eqnarray*}
    \item Matched knowledge: When $\pi=1/2$, $p_+=p_-=1$, and $\theta_-^e=-\theta_M=-\theta_+^e$,we have
\begin{eqnarray*}
    P_+^*= 1-\Phi \left(\frac{-\|\theta_M\|_2}{\sqrt{\sigma^2+\sigma_{e+}^2}}\right).
\end{eqnarray*}
\end{itemize}
\textbf{Regime 2: $k\sigma_M^2\ll \sigma^2$.} 
\begin{itemize}
    \item Matched knowledge:
    Recall that
$$\theta_+\sim N\left(\frac{\sigma_+^2\theta_M+\sigma_M^2\sum_{y_i=+1}x_i}{\sigma_+^2+\#(y_i=+1)\sigma^2_M},\frac{\sigma_+^2\sigma_M^2}{\sigma_+^2+\#(y_i=+1)\sigma^2_M}I\right)\triangleq N(\hat{\theta}_+,\sigma^2_{\theta_+}I),$$
and
$$\theta_-\sim N\left(\frac{\sigma_M^2\sum_{y_i=-1}x_i-\sigma_-^2\theta_M}{\sigma_-^2+\#(y_i=-1)\sigma_M^2},\frac{\sigma_M^2\sigma_-^2}{\sigma_-^2+\#(y_i=-1)\sigma_M^2}I\right)\triangleq N(\hat{\theta}_-,\sigma^2_{\theta_-}I).$$
When $k\sigma_M^2\ll \sigma^2$,
\begin{eqnarray*}
    \hat\theta_+\rightarrow \theta_M,\qquad \hat\theta_-\rightarrow-\theta_M.
\end{eqnarray*}
As a result,
\begin{eqnarray*}
    P(correct|y=+1,\Theta, \{(x_i,y_i)\},M)\rightarrow1-\Phi \left(-\frac{\|\theta_M\|_2}{\sqrt{\sigma^2+\sigma_{e+}^2}}\right).
\end{eqnarray*}
\item Contradict knowledge: One can follow similar arguments as the above to obtain
\begin{eqnarray*}
    P(correct|y=+1,\Theta, \{(x_i,y_i)\},M)\rightarrow 1-\Phi \left(\frac{\|\theta_M\|_2}{\sqrt{\sigma^2+\sigma_{e+}^2}}\right).
\end{eqnarray*}
\end{itemize}

\end{proof}

\subsection{Proof of Proposition \ref{prop:imbalance}}\label{sec:appendix:proof:imbalance}
\begin{proof}[Proof of Proposition \ref{prop:imbalance}]
One can directly obtain the result by discussing the value of $\pi$ in $P^*_+$ and $P^*_-$.

\end{proof}

\subsection{Proof of Proposition \ref{prop:noise}}\label{sec:appendix:proof:noise}
\begin{proof}[Proof of Proposition \ref{prop:noise}]

Following the definition of $\hat\theta_+$ and $\hat\theta_-$, when taking $k\rightarrow\infty$ and $\theta_+^e=\theta_M=-\theta_-^e$, we  obtain
 \begin{eqnarray*}
   \hat\theta_-+\hat\theta_+&\rightarrow& (1+p_+^e-p_-^e)\theta_+^e+(1-p_+^e+p_-^e)\theta_-^e=2(p_+^e-p_-^e)\theta_M,\\
    \hat\theta_--\hat\theta_+&\rightarrow&
    (1-p_+^e-p_-^e)(\theta_+^e-\theta_-^e)=2(1-p_+^e-p_-^e)\theta_M.
\end{eqnarray*}   
Taking the above into the formula of $P(correct|y=+1)$ in Theorem \ref{thm:icl_acc}, we obtain

\begin{eqnarray*}
    P(correct|y=+1, \{(x_i,y_i)\},M)&=&1-\Phi\left(
    \frac{\theta^e_+-(\hat\theta_++\hat\theta_-)/2}{\sqrt{(\sigma^2+\sigma^2_{e+})}}\frac{\hat{\theta}_--\hat{\theta}_+}{\|\hat{\theta}_--\hat{\theta}_+\|}+\frac{(\sigma^2+\sigma^2_{\theta})\log z_k}{\|\hat{\theta}_--\hat{\theta}_+\|\sqrt{(\sigma^2+\sigma^2_{e+})}}
    \right)\\
    &=&1-\Phi\left(C_1(1-p_+^e+p_-^e)\|\theta_M\|\text{sign}(1-p_+^e-p_-^e)+C_2\frac{\log z_k}{|1-p_+^e-p_-^e|}
    \right),
\end{eqnarray*}
and
    \begin{eqnarray*}
    P(correct|y=-1, \{(x_i,y_i)\},M)&=&
    \Phi\left(\frac{(\theta^e_--(\hat{\theta}_++\hat{\theta}_-)/2)^T}{\sqrt{\sigma^2+\sigma^2_{e-}}}\frac{\hat{\theta}_--\hat{\theta}_+}{\|\hat{\theta}_--\hat{\theta}_+\|_2}+\frac{(\sigma^2+\sigma^2_{\theta})\log z_k}{\|\hat{\theta}_--\hat{\theta}_+\|\sqrt{(\sigma^2+\sigma^2_{e+})}}\right)\\
    &=&
    \Phi\left(-C_1(1+p_+^e-p_-^e)\text{sign}(1-p_+^e-p_-^e) +C_2\frac{\log z_k}{|1-p_+^e-p_-^e|}\right),
\end{eqnarray*}
where $C_1=2\frac{\|\theta_M\|}{\sqrt{(\sigma^2+\sigma^2_{e+})}}>0, C_2=\frac{\sigma^2}{\sqrt{(\sigma^2+\sigma^2_{e+})}\|\theta_M\|}>0$.
\end{proof}

\subsection{Proof of Theorem \ref{them: dependent}} \label{app: proof dependent}

\begin{proof}[Proof of Theorem \ref{them: dependent}]
    During pre-training, since there are $k$ examples and one test sample in the prompt, the fraction of positive labels can only take values in the form of $i/(k+1)$ for $i=0,\ldots,k$, rather than a continuous variable in $[0,1]$. As a result, to connect the distribution of $frac$ with the Beta distribution, we assume $k+1$ is odd and denote $B$ as a random variable following $\text{Beta}(\alpha,\beta)$. Then we set the following:
    \begin{eqnarray*}
        P(frac=i/(k+1)|M)=\begin{cases}
            P\left(B<\frac{1}{2(k+1)}\right) & i=0\\
            P\left( \frac{2i-1}{2(k+1)}\leq B<\frac{2i+1}{(2k+1)} \right) & 0<i<k+1\\
            P\left( B\geq \frac{2k+1}{2(k+1)} \right) & i=k+1
        \end{cases}.
    \end{eqnarray*}
    In addition, we assume that all $y_i$s and $y$ have equal chance of being positive.
    
    Based on our assumption, when LLM learns from the pre-training data, it can exactly learn the distribution of $frac$, and use the likelihood to make a decision when receiving the testing data.

    In the testing stage, when receiving $\{(x_i,y_i)\}_{i\in[k]}$ and $x$. From the definition of conditional probability, we know that
\begin{eqnarray*}
    P(y=+1|x,\{(x_i,y_i)\},M)=\frac{P((x,y=+1)|\{(x_i,y_i)\},M)}{P((x,y=+1)|\{(x_i,y_i)\},M)+P((x,y=-1)|\{(x_i,y_i)\},M)},
\end{eqnarray*}
where 
\begin{equation*}
    \begin{aligned}
        &P((x,y=+1)|\{(x_i,y_i)\},M)=P(x|y=+1, \{(x_i,y_i)\},M)P(y=+1|\{(x_i,y_i)\},M).
    \end{aligned}
\end{equation*}
    From the above, we need to figure out the following quantity:
    \begin{eqnarray*}
        &&P((x,y=+1)|\{(x_i,y_i)\},M)\\
        &=&P(x|y=+1, {(x_i,y_i)},M)P(y=+1|\{(x_i,y_i)\},M)\\
        &=&P(x|y=+1, \{(x_i,y_i)\},M)P\left(frac=\frac{1+\#(y_i=+1)}{k+1}\bigg|\{(x_i,y_i)\},M\right),
    \end{eqnarray*}
    where
    \begin{eqnarray}
&&P\left(frac=\frac{1+\#(y_i=+1)}{k+1}\bigg|\{(x_i,y_i)\},M\right)\nonumber\\
&=& \frac{P\left(frac=\frac{1+\#(y_i=+1)}{k+1},\{(x_i,y_i)\}\bigg|M\right)}{P\left(frac=\frac{1+\#(y_i=+1)}{k+1},\{(x_i,y_i)\}\bigg|M\right)+P\left(frac=\frac{\#(y_i=+1)}{k+1},\{(x_i,y_i)\}\bigg|M\right)}.\label{eqn:frac}
\end{eqnarray}
To calculate $P(frac=(1+\#(y_i=+1))/{(k+1)},\{(x_i,y_i)\}|M)$, when $frac=(1+\#(y_i=+1))/{(k+1)}$, it means that there are $1+\#(y_i=+1)$ examples (and the query) which have a positive label. Given a total of $k+1$ data, there are $C_{k+1}^{1+\#(y_i=+1)}$ different combinations. As a result, for a fixed $\{(x_i,y_i)\}$, we have
$$P\left(frac=\frac{1+\#(y_i=+1)}{k+1},\{(x_i,y_i)\}\bigg| M\right)=\frac{1}{C_{k+1}^{1+\#(y_i=+1)}} P\left( frac=\frac{1+\#(y_i=+1)}{k+1} \bigg|M\right).$$
Similarly, we obtain that
$$P\left(frac=\frac{\#(y_i=+1)}{k+1},\{(x_i,y_i)\}\bigg|M\right)=\frac{1}{C_{k+1}^{\#(y_i=+1)}} P\left( frac=\frac{\#(y_i=+1)}{k+1}\bigg|M \right).$$
Taking the above into (\ref{eqn:frac}), it becomes
\begin{eqnarray*}
&&P\left(frac=\frac{1+\#(y_i=+1)}{k+1}\bigg|\{(x_i,y_i)\},M\right)\nonumber\\
&=& \frac{P\left(frac=\frac{1+\#(y_i=+1)}{k+1}\bigg|M\right)C_{k+1}^{\#(y_i=+1)}}{P\left(frac=\frac{1+\#(y_i=+1)}{k+1}\bigg|M\right)C_{k+1}^{\#(y_i=+1)}+P\left(frac=\frac{\#(y_i=+1)}{k+1}\bigg|M\right)C_{k+1}^{1+\#(y_i=+1)}}\\
&=& \frac{1}{1+\frac{P\left(frac=\frac{\#(y_i=+1)}{k+1}|M\right)}{P\left(frac=\frac{1+\#(y_i=+1)}{k+1}|M\right)}\frac{k-\#(y_i=+1)+1}{1+\#(y_i=+1)}}.
\end{eqnarray*}
To further look at the exact value of $P\left(frac=\frac{1+\#(y_i=+1)}{k+1}\bigg|\{(x_i,y_i)\},M\right)$, we need to figure out $P(frac=i/(k+1)|M)$ using the Beta distribution. Recall that the probability density function $f$ of Beta$(\alpha,\beta)$ satisfies
\begin{eqnarray*}
    f(u)=\frac{u^{\alpha-1}(1-u)^{\beta-1}}{B(\alpha,\beta)}
\end{eqnarray*}
for Beta function $B(\alpha,\beta)$. Recall that we assume that $\alpha/\beta=\pi/(1-\pi)$ for some $\pi\in (0,1)$, and both $\alpha,\beta\rightarrow\infty$. When $u<(\alpha-1)/(\alpha+\beta-2)\approx \pi$, we have $f$ is an increasing function in $u$, otherwise $f$ is decreasing. This implies that the largest probability of $frac$ may be taken from $P(frac=(\lfloor \pi(k+1) \rfloor+1)/(k+1))$, $P(frac=\lfloor \pi(k+1) \rfloor/(k+1))$ or $P(frac=(\lfloor \pi(k+1) \rfloor-1)/(k+1))$.

When $frac<(\lfloor \pi(k+1) \rfloor-1)/(k+1)$, when $\alpha$ and $\beta$ are large enough, one can obtain that
\begin{eqnarray*}
    \frac{P\left(frac=\frac{\#(y_i=+1)}{k+1}|M\right)}{P\left(frac=\frac{1+\#(y_i=+1)}{k+1}|M\right)}\frac{k-\#(y_i=+1)+1}{1+\#(y_i=+1)}\rightarrow 0,
\end{eqnarray*}
which implies that
\begin{eqnarray*}
P\left(frac=\frac{1+\#(y_i=+1)}{k+1}\bigg|\{(x_i,y_i)\},M\right)=\frac{1}{1+\frac{P\left(frac=\frac{\#(y_i=+1)}{k+1}|M\right)}{P\left(frac=\frac{1+\#(y_i=+1)}{k+1}|M\right)}\frac{k-\#(y_i=+1)+1}{1+\#(y_i=+1)}}\rightarrow 1.
\end{eqnarray*}
Similarly, when $frac>(\lceil \pi(k+1) \rceil+1)/(k+1)$,
\begin{eqnarray*}
P\left(frac=\frac{1+\#(y_i=+1)}{k+1}\bigg|\{(x_i,y_i)\},M\right)=\frac{1}{1+\frac{P\left(frac=\frac{\#(y_i=+1)}{k+1}|M\right)}{P\left(frac=\frac{1+\#(y_i=+1)}{k+1}|M\right)}\frac{k-\#(y_i=+1)+1}{1+\#(y_i=+1)}}\rightarrow 0.
\end{eqnarray*}
Finally, we put $P\left(frac=\frac{1+\#(y_i=+1)}{k+1}\bigg|\{(x_i,y_i)\},M\right)$ into $P(y=+1|x,\{(x_i,y_i),M\})$:
{\tiny
    \begin{eqnarray*}
        &&P(y=+1|x,\{(x_i,y_i)\},M)\\
        &=&\frac{P((x,y=+1)|\{(x_i,y_i)\},M)}{P((x,y=+1)|\{(x_i,y_i)\},M)+P((x,y=-1)|\{(x_i,y_i)\},M)}\\
        &=&\frac{P(x|y=+1, \{(x_i,y_i)\}|M)P(frac=\frac{1+\#(y_i=+1)}{k+1}|\{(x_i,y_i)\},M)}{P(x|y=+1, \{(x_i,y_i),M\})P(frac=\frac{1+\#(y_i=+1)}{k+1}|\{(x_i,y_i)\},M)+P(x|y=-1, \{(x_i,y_i)\},M)P(frac=\frac{\#(y_i=+1)}{k+1}|\{(x_i,y_i)\},M)}.
    \end{eqnarray*}
    }
When $P(x|y=+1,\{(x_i,y_i),M\})$ and $P(x|y=-1,\{(x_i,y_i),M\})$ are both bounded and bounded away from zero, we have
\begin{equation*}
    P(y=+1|x,\{(x_i,y_i)\},M)\rightarrow \begin{cases}
        1 &\text{if }\frac{\#(y_i=+1)}{k+1}<\frac{\lfloor \pi(k+1)\rfloor-1}{k+1}\\
        0 &\text{if }\frac{\#(y_i=+1)}{k+1}>\frac{\lceil \pi(k+1)\rceil+1}{k+1}
    \end{cases},
\end{equation*}
which completes the proof.
\end{proof}
\newpage
\section{Simulation setups}\label{sec:appendix:simulation setup}
In this section, we provide details of experimental setups for simulation.

\textbf{Model structure.}~~Following previous work \citep{garg2022can}, we use a decoder-only Transformer architecture \citep{vaswani2017attention} from the GPT-2 family \citep{radford2019language}. This model consists of 2 layers, 8 attention heads, and a 256-dimensional
embedding space.

\textbf{Pre-training.}~~We train the model using a cross-entropy loss function for binary classification. We sample a batch of random prompts at each training step and update the model through a gradient update. We train with a batch size of 64 and for 50k steps. This training is done from scratch, that is, we do not fine-tune a pre-trained language model, nor do we train on actual text. Following previous work \citep{garg2022can}, we also use curriculum learning \citep{bengio2009curriculum, elman1993learning}. In particular, we start with a shorter length of prompts (10 input-output pairs) and increase the length by 2 every 2000 training steps. For the other hyperparameters, e.g., learning rate, we use the default values as in \citep{garg2022can}.

\textbf{Pre-training data.}~~We follow the data generation model in Section \ref{sec:assumption}. We first select label $y\in \{+1,-1\}$ with probability $\pi$ (positive probability). Then for inputs with positive labels, we first sample a mean value $\theta_+$ from a Gaussian distribution $N(\theta_M,\sigma_M^2I)$, and then sample data $x$ from Gaussian distribution $N(\theta_+, \sigma^2I)$; similar for the inputs with label $-1$, we sample $\theta_-$ from $N(-\theta_M, \sigma_M^2I)$, and sample $x$ from $N(\theta_-,\sigma^2I)$. Specifically, we let $\theta_M=0.5\mathbf{1}, \sigma_M^2=\sigma^2=1$. During the pre-training, to ensure the transformer can learn the population information rather than overfitting a particular set of data, we sample a new pair of $(\theta_+,\theta_-)$ for each iteration and generate corresponding sample pair $(x_i,y_i)$. 

\textbf{Computation resources.}~~Both simulations and real-world experiments are running on a server with 8 Nvidia RTX A6000 GPU (48G GPU memory each) and 32 AMD EPYC 7302 16-Core Processors.

\section{Additional Experiment Results}\label{sec:appendix:experiments}
\subsection{Simulation for Independent Exampels}\label{sec:appendix:experiments:independent_simulation}

Figure \ref{fig:contradict1_}, \ref{fig:contradict2}, \ref{fig:contradict4} represents additional results corresponding to the contradict knowledge setting in Figure \ref{fig:contradict1}. The observations are similar to Figure \ref{fig:contradict1}.

\begin{figure}[h]
    \centering
    \begin{minipage}{0.32\textwidth}
        \includegraphics[width=\linewidth]{figures/contradict1.0.png} 
        \caption{$\sigma^2=1$}
        \label{fig:contradict1_}
    \end{minipage}\hfill
    \begin{minipage}{0.32\textwidth}
        \includegraphics[width=\linewidth]{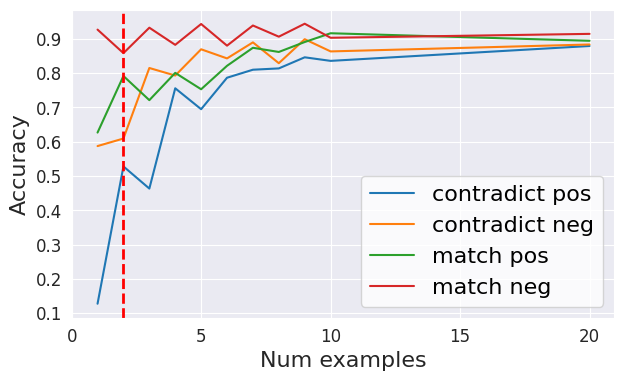} 
        \caption{$\sigma^2=2$}
        \label{fig:contradict2}
    \end{minipage}\hfill
    \begin{minipage}{0.32\textwidth}
        \includegraphics[width=\linewidth]{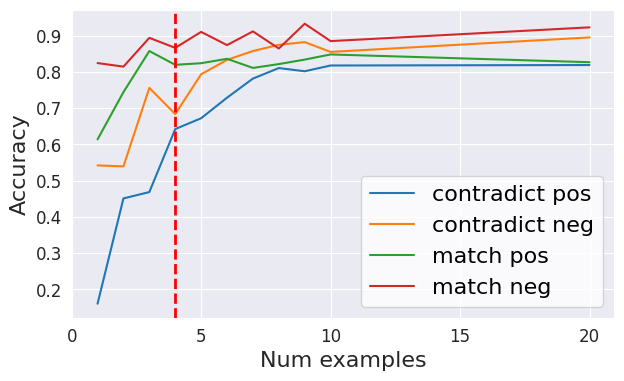} 
        \caption{$\sigma^2=4$}
        \label{fig:contradict4}
    \end{minipage}
\end{figure}

\subsection{Mean Reversion}\label{sec:appendix:experiment:dependent}
We further pre-train GPT models with different fractions ($frac=0.2,0.5,0.8$) and test the posterior distribution of labels when the fraction of positive labels within examples varies. We do not add noise and keep other settings unchanged. We observe a dramatic change around 0.2,0.5,0.8 respectively in Figure \ref{fig:dep frac}, and these figures directly verify our results: in Theorem \ref{them: dependent}, the cutting point are $0.2,0.5,0.8$ respectively in the three settings.

\begin{figure}[!ht]
    \centering
        \includegraphics[width=\linewidth]{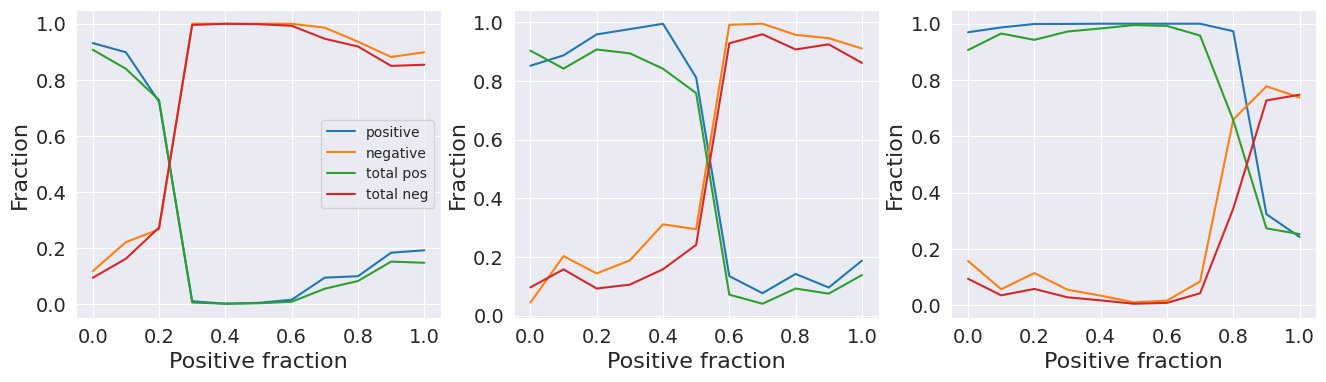} 
        \caption{Pre-training with a fraction of positive 0.2.}
        \label{fig:dep frac}
\end{figure}

\end{document}